\documentclass[acmsmall,screen,nonacm]{acmart}

\setcopyright{none}
\makeatletter
\@ifundefined{hyxmp@parse@acmart}{}{\let\hyxmp@parse@acmart\relax}
\makeatother

\usepackage{amsmath}
\usepackage{mathtools}
\usepackage{booktabs}
\usepackage{tabularx}
\usepackage{longtable}
\usepackage{array}
\usepackage{multirow}
\usepackage{enumitem}
\usepackage{graphicx}
\usepackage{xcolor}
\usepackage{tikz}
\usepackage{url}
\usepackage{microtype}
\usepackage{algorithm}
\usepackage{algpseudocode}
\usepackage{placeins}
\usetikzlibrary{arrows.meta,positioning,fit,calc,shapes.geometric,backgrounds}

\hypersetup{
  hidelinks,
  keeppdfinfo=true,
  pdftitle={Runtime Authorization of Self-Generated Subgoals in Long-Horizon Tool-Using AI Agents},
  pdfauthor={Genliang Zhu; Chu Wang},
  pdfsubject={Runtime authorization of self-generated subgoals in long-horizon tool-using AI agents},
  pdfcreator={LaTeX}
}

\setlist[itemize]{leftmargin=1.25em}
\setlist[enumerate]{leftmargin=1.45em}
\newcolumntype{L}[1]{>{\raggedright\arraybackslash}p{#1}}
\newcolumntype{Y}{>{\raggedright\arraybackslash}X}

\definecolor{gtBlue}{HTML}{1D4ED8}
\definecolor{gtBlueFill}{HTML}{EFF6FF}
\definecolor{gtGreen}{HTML}{15803D}
\definecolor{gtGreenFill}{HTML}{F0FDF4}
\definecolor{gtAmber}{HTML}{B45309}
\definecolor{gtAmberFill}{HTML}{FFFBEB}
\definecolor{gtPurple}{HTML}{6D28D9}
\definecolor{gtPurpleFill}{HTML}{F5F3FF}
\definecolor{gtRed}{HTML}{B91C1C}
\definecolor{gtRedFill}{HTML}{FEF2F2}

\newtheorem{definition}{Definition}
\newtheorem{lemma}{Lemma}
\newtheorem{theorem}{Theorem}
\newtheorem{proposition}{Proposition}

\newcommand{\mcode}[1]{\ifmmode\text{\normalfont\ttfamily #1}\else{\normalfont\ttfamily #1}\fi}
\newcommand{\Effects}{\ensuremath{\mathbb{E}}}

\newcommand{\Allow}{\textsc{Allow}}
\newcommand{\Deny}{\textsc{Deny}}
\newcommand{\Indeterminate}{\textsc{Indeterminate}}
\newcommand{\Reauthorize}{\textsc{Reauthorize}}

\newcommand{\GoalTransitionContextualTemplateCount}{48}
\newcommand{\GoalTransitionStructuralSchemaCount}{25}
\newcommand{\GoalTransitionMatchedCaseCount}{96}
\newcommand{\GoalTransitionBenignCaseCount}{48}
\newcommand{\GoalTransitionDriftedCaseCount}{48}

\newcommand{\GoalTransitionDecisionCount}{576}

\newcommand{\GoalTransitionCompleteBenignCompletionCount}{48}
\newcommand{\GoalTransitionMemorylessUnsafeCommitCount}{33}
\newcommand{\GoalTransitionMemorylessDriftedCaseCount}{48}
\newcommand{\GoalTransitionHistoryPreventiveBlockCount}{37}
\newcommand{\GoalTransitionHistoryDriftedCommittedCount}{11}
\newcommand{\GoalTransitionCompleteDeniedCount}{43}
\newcommand{\GoalTransitionCompleteReauthorizationCount}{5}
\newcommand{\GoalTransitionBlockedBeforeEffectCount}{38}
\newcommand{\GoalTransitionInvalidOrderAfterSafePrefixCount}{5}
\newcommand{\GoalTransitionCheckerMutantCount}{2692}
\newcommand{\GoalTransitionResealedReceiptMutantCount}{1631}
\newcommand{\GoalTransitionSemanticallyReplayedReceiptCount}{867}
\newcommand{\GoalTransitionProtocolCheckCount}{65}
\newcommand{\GoalTransitionProtocolPositiveCheckCount}{14}
\newcommand{\GoalTransitionProtocolNegativeCheckCount}{51}
\newcommand{\GoalTransitionProtocolGateCount}{43}
\newcommand{\GoalTransitionCheckerMutationFamilyCount}{48}
\newcommand{\GoalTransitionIndependentProfileLimitMutationCount}{4}

\newcommand{\GoalTransitionRuntimeCaseCount}{32}
\newcommand{\GoalTransitionRuntimeDispatchCount}{258}
\newcommand{\GoalTransitionRuntimeDecisionMatchCount}{32}
\newcommand{\GoalTransitionRuntimeReceiptChainCount}{32}
\newcommand{\GoalTransitionCrossRuntimeRecordCount}{16}
\newcommand{\GoalTransitionCrossRuntimeStateCount}{16}
\newcommand{\GoalTransitionPerformanceSchemaCount}{25}
\newcommand{\GoalTransitionPerformanceSampleCount}{30000}
\newcommand{\GoalTransitionBoundedStateCount}{340}
\newcommand{\GoalTransitionBoundedTransitionCount}{419}
\newcommand{\GoalTransitionTransitionMedianMicroseconds}{83.584}
\newcommand{\GoalTransitionTransitionPercentileNinetyFiveMicroseconds}{131.000}
\newcommand{\GoalTransitionTransitionPercentileNinetyNineMicroseconds}{225.166}

\newcommand{\GoalTransitionEffectAuthorizationMedianMicroseconds}{22.292}
\newcommand{\GoalTransitionEffectAuthorizationPercentileNinetyFiveMicroseconds}{31.292}
\newcommand{\GoalTransitionEffectAuthorizationPercentileNinetyNineMicroseconds}{57.167}
\newcommand{\GoalTransitionEffectCommitMedianMicroseconds}{42.375}
\newcommand{\GoalTransitionEffectCommitPercentileNinetyFiveMicroseconds}{60.250}
\newcommand{\GoalTransitionEffectCommitPercentileNinetyNineMicroseconds}{147.000}
\newcommand{\GoalTransitionTransitionReceiptBytes}{1348}

\newcommand{\GoalTransitionEffectAuthorizationReceiptBytes}{847}
\newcommand{\GoalTransitionEffectCommitReceiptBytes}{1673}
\newcommand{\GoalTransitionRetainedStateAfterTransitionBytes}{5470}
\newcommand{\GoalTransitionRetainedStateAfterEffectBytes}{8866}
\newcommand{\GoalTransitionScaleCellCount}{129}

\newcommand{\GoalTransitionScaleWarmupExecutionCount}{6450}
\newcommand{\GoalTransitionScaleMeasuredExecutionCount}{25800}
\newcommand{\GoalTransitionScaleMemoryExecutionCount}{129}

\newcommand{\GoalTransitionScaleHeapImmediateMedianBytes}{687016}

\newcommand{\GoalTransitionScaleHeapAfterGarbageCollectionMedianBytes}{84680}

\newcommand{\GoalTransitionScaleMaximumResidentSetDeltaMedianKilobytes}{288}

\newcommand{\GoalTransitionScaleAllowMaximumPercentileNinetyNineMicroseconds}{26739.833}
\newcommand{\GoalTransitionScaleAllowMaximumPercentileNinetyNineCell}{\mcode{F-256}}
\newcommand{\GoalTransitionScaleRelativeBlockMedianMadMedianPercent}{1.43}
\newcommand{\GoalTransitionScaleRelativeBlockMedianMadAtMostFivePercentCellCount}{123}
\newcommand{\GoalTransitionScaleGoalPopulationCapRetainedStateBytes}{296080}
\newcommand{\GoalTransitionScaleJoinOutputCapRetainedStateBytes}{117996}
\newcommand{\GoalTransitionScaleGoalPopulationCapCompleteReceiptChainBytes}{77727}
\newcommand{\GoalTransitionScaleMaximumDepthCapCompleteReceiptChainBytes}{78967}
\newcommand{\GoalTransitionScalePrecedenceEdgeCapTargetReceiptBytes}{4933}

\title[Runtime Authorization of Self-Generated Subgoals]
{Runtime Authorization of Self-Generated Subgoals in Long-Horizon Tool-Using AI Agents}

\author{Genliang Zhu}
\affiliation{%
  \institution{Accentrust}
  \city{Vancouver}
  \country{Canada}
}
\affiliation{%
  \institution{Georgia Institute of Technology}
  \department{School of Computer Science}
  \city{Atlanta}
  \state{Georgia}
  \country{USA}
}
\email{research@accentrust.com}

\author{Chu Wang}
\affiliation{%
  \institution{Accentrust}
  \city{Vancouver}
  \country{Canada}
}
\affiliation{%
  \institution{University of Illinois Urbana-Champaign}
  \department{Department of Computer Science}
  \city{Urbana}
  \state{Illinois}
  \country{USA}
}
\email{wally@accentrust.com}

\begin{document}

\begin{abstract}
Long-horizon, tool-using AI agents create subgoals, replan, delegate work, and
compose sibling results.  This creates an authorization gap: per-tool
permission checks cannot establish that a changing goal graph remains within
the principal-approved task.  We study this gap in a finite structured domain
with one principal and one authorization root.  Each proposed goal-graph
mutation carries a
version-bound witness that its continuation traces, resources, obligations,
active invariants, and closing condition refine the active root task contract;
each protected effect is rechecked at an atomic commit boundary.  Free-form
goal text supplies no authority.

We prove root trace-policy and modeled forbidden-state preservation under
explicit mediation, abstraction, freshness, and atomicity assumptions, plus
conditional root-success preservation, a separation result for memoryless
allowlists, exact finite-domain decidability, and universal-safety monotonicity
under sound abstraction refinement.  An executable model explores
\GoalTransitionBoundedStateCount{} states and
\GoalTransitionBoundedTransitionCount{} transitions.  We evaluate
\GoalTransitionMatchedCaseCount{} matched cases spanning
\GoalTransitionStructuralSchemaCount{} structural schemas; 24 of
\GoalTransitionContextualTemplateCount{} templates map to frozen identifiers
from four public benchmark artifacts.  The
complete mechanism commits zero forbidden states in
\GoalTransitionDriftedCaseCount{} drifted cases and completes all
\GoalTransitionBenignCaseCount{} benign counterparts.  Two public upstream
runtime paths execute \GoalTransitionRuntimeDispatchCount{} native dispatches
across \GoalTransitionRuntimeCaseCount{} cases, with every case-level decision
and receipt chain matching.  In a separately frozen host-local study, all
\GoalTransitionScaleCellCount{} synthetic one-factor-at-a-time cells match
their fixed decision and reason; the resulting curves characterize the finite
profile rather than production scalability.  A history-aware continuation comparator blocks
every modeled bad trace prefix yet commits all operations in
\GoalTransitionHistoryDriftedCommittedCount{} cases whose violations lie in
typed resources, freshness, or explicit-join evidence outside its trace
projection.  Within the registered structured domains, runtime authorization
preserves useful replanning while preventing self-generated subgoals from
becoming a source of new authority.
\end{abstract}



\maketitle

\section{Introduction}
\label{sec:introduction}

Tool-using AI agents increasingly execute tasks that extend beyond one model
response or one API call.  A long-running agent may decompose a request,
delegate work, revise its plan after an observation, recover from a failed
operation, and combine the outputs of several subgoals before producing an
external effect.  These capabilities improve task completion, but they also
change the unit that authorization must govern.  A permission check on the
next tool call answers whether that call is locally admissible.  It does not,
by itself, answer whether the growing goal graph still implements the task the
principal authorized.

This distinction matters whenever policy is non-compositional.  Two writes can
each target an allowed field while their joint state violates a separation
constraint.  A recovery action can use an allowed tool yet widen the resource
scope from one account to every account.  A delegated task can preserve the
parent's vocabulary while replacing its completion condition with an easier
proxy.  A stale plan can remain locally valid after the policy or relevant
environment state has changed.  In each case, an allowlist or independent
per-call classifier can accept every action even though the authorized task
has been expanded through planning.

The classical literature already supplies the relevant foundations.  Execution
monitoring characterizes policies enforceable from traces
\cite{schneider2000enforceable,ligatti2005editautomata,leucker2009runtimeverification}; hierarchical
planning gives formal accounts of task decomposition and refinement
\cite{erol1996htncomplexity,georgievski2015htnoverview}; and safety and liveness separate finite
bad prefixes from eventual progress requirements
\cite{alpern1985liveness}.  Recent peer-reviewed agent benchmarks show that
realistic tool use is stateful, multi-step, and vulnerable to adversarial or
policy-violating trajectories
\cite{ruan2024toolemu,debenedetti2024agentdojo,yao2025taubench,trivedi2024appworld,tur2025safearena,jiang2026agentlab,li2026multiturn}.
The open systems question is therefore not whether sequence-aware policies
exist.  It is how to connect a dynamically changing, agent-produced goal graph
to a reference monitor so that planning changes become authorization events
and protected effects remain inside the original structured task contract.

\subsection{Problem statement}

We study one principal, one approved root task contract, and a finite
structured effect domain.  The root contract fixes required and forbidden
state predicates, resource bounds, outstanding obligations, temporal limits,
and a prefix-closed language of permitted effect traces.  The planner and its
generated natural-language explanations are untrusted.  Before a proposed
child, replacement, delegation, or same-root join becomes active, the agent
must submit a structured goal description and a checkable transition witness.
The witness establishes that the proposed continuation is no broader than the
current authorized continuation, preserves active invariants and obligations,
and was derived for the current policy and state versions.  Every protected
effect is then mediated against the active goal and the residual permitted
trace language.

The paper does not infer authority from free-form text.  An upstream component
may translate a user request into a structured contract, but the principal or
a trusted policy service must approve the structure used for enforcement.
This boundary keeps the security claim decidable and auditable: the mechanism
checks refinement in a registered finite domain rather than claiming general
natural-language entailment.

\subsection{Why local authorization is insufficient}

Consider a task that permits changing either a deployment artifact or its
traffic target during diagnosis, but forbids publishing unreviewed code to the
production target.  The first update is locally permitted, as is the second.
After the first update, however, the second must be rejected unless a review
obligation has been discharged.  A memoryless action allowlist observes the
same second action in both the safe and unsafe histories and must return the
same decision.  A final-state checker can detect the violation only after both
effects have committed.  The missing state is not another textual warning; it
is the authorized composition of the current prefix and the still-reachable
future effects.  Section~\ref{sec:counterexample} gives the executable minimal
case and its matched benign trace.

\subsection{Approach}

We represent the approved task and every proposed subgoal in one typed effect
and state vocabulary.  A goal node records its parent, success predicate,
resources, obligations, active frames, policy/state versions, and a finite set
of maximal completion traces whose prefix closure denotes the remaining
continuations.  An ordinary transition is checked along one fresh
leaf-to-root witness path; sibling continuations are composed only by an
explicit join request whose admissible interleavings are enumerated.
Authorization checks the proposed edit against every residual ancestor
language on that path.  A definite counterexample returns \Deny{}; an
unsupported or resource-exhausting decision returns \Indeterminate{}; and a
relevant policy or observed-state version change suspends the execution until
\Reauthorize{}.  Changing the root contract instead starts a new execution
epoch.  Only \Allow{} makes the transition active.

The mechanism is external to the planner.  It therefore does not depend on a
particular prompting method, model provider, or disclosure of private chain of
thought.  Its trusted computing base is limited to the contract registry,
transition checker, effect normalizer, state/version source, receipt store, and
commit adapter.  The model, planner, retrieved content, worker agents, and
free-form goal text remain untrusted.

\subsection{Contributions}

This paper makes six contributions:

\begin{enumerate}
  \item It identifies a precise authorization gap between locally permitted
  tool calls and globally permitted goal evolution, and supplies an executable
  matched counterexample in which identical current actions require different
  decisions because their prefixes differ.

  \item It defines a structured semantics for a root task contract, a
  versioned goal graph, transition witnesses, and residual permitted trace
  languages.  Goal creation, replacement, delegation, and same-root joining
  are treated as authorization-relevant state changes rather than planner-only
  bookkeeping.

  \item It proves root-policy safety under explicit mediation, abstraction,
  freshness, and commit assumptions; conditional preservation of required
  terminal predicates; a standard insufficiency result for memoryless local
  checks; a terminating exact decision procedure for registered finite domains;
  and monotonicity of the universal-safety subcheck under sound abstraction
  refinement.

  \item It implements the transition protocol, an independently implemented
  trace checker, bounded state exploration, targeted security mutants, and
  six specified comparison mechanisms.  Each comparison receives the exact
  input projection required by its definition; the strongest trace comparison
  receives the root trace policy, effect catalog and transition semantics,
  forbidden and required-terminal predicates, current state, complete committed
  prefix, proposed continuations, and proposed goal success predicate.

  \item It evaluates matched benign and drifted workflows rooted in public
  task sources through two independently developed public tool-execution
  paths.  Outcomes are fixed by structured state and effect validators;
  post-commit detection is reported separately from prevention at the commit
  boundary.

  \item A separately frozen secondary study probes the finite reference
  envelope over \GoalTransitionScaleCellCount{} synthetic
  one-factor-at-a-time cells.  It tests direct representation limits and
  internal enumeration guards while retaining host-local latency, serialized
  evidence size, retained-state size, and runtime memory indicators.
\end{enumerate}

The empirical claims in this manuscript are generated from the retained
receipts and fixed denominators described in
Section~\ref{sec:evaluation-design}; exact results appear only after the
artifact verifier reproduces them.  The primary local-cost snapshot contains
\GoalTransitionPerformanceSampleCount{} timed transition, effect-authorization,
and commit operations over 25 structural representatives.  The separate scale
study contains \GoalTransitionScaleMeasuredExecutionCount{}
\mcode{authorizeTransition} samples.  These denominators answer different
questions and are never pooled.

\subsection{Organization}

Section~\ref{sec:counterexample} presents the separating execution and threat
model.  Section~\ref{sec:related-work} positions the mechanism against access
control, runtime enforcement, planning verification, and agent safety.
Section~\ref{sec:formalization} defines the model and results.
Sections~\ref{sec:protocol} and~\ref{sec:implementation} present the protocol
and reference realization.  Sections~\ref{sec:evaluation-design}
and~\ref{sec:results} give the registered study and results.
Sections~\ref{sec:discussion}--\ref{sec:conclusion} state deployment
implications, validity conditions, and conclusions.

\section{Motivating Counterexample and Threat Model}
\label{sec:counterexample}

\subsection{A matched pair with identical local permissions}

The smallest separating example uses two Boolean state variables:
\(u\), indicating that an unreviewed artifact is selected, and \(p\),
indicating that production traffic targets the selected artifact.  The root
task is to diagnose a failed deployment.  It permits the normalized effects
\(\mcode{select(unreviewed)}\), \(\mcode{route(production)}\), and
\(\mcode{record(review)}\) as individual operations, but forbids the reachable
state \(u\land p\) unless the review obligation has first been discharged.

The benign and drifted executions share the same root task and tool affordance:

\begin{align*}
  t_{\mathrm{safe}} &={}
    \langle \mcode{select(unreviewed)},\mcode{record(review)},
    \mcode{route(production)}\rangle,\\
  t_{\mathrm{drift}} &={}
    \langle \mcode{select(unreviewed)},\mcode{route(production)}\rangle.
\end{align*}

Every symbol belongs to the static allowlist.  At the second trace's final
step, a memoryless checker sees the same locally permitted
\(\mcode{route(production)}\) action that appears in the safe trace.  The
correct decision nevertheless depends on the committed prefix and the
undischarged obligation.  The drifted plan therefore separates membership in a
set of individually permitted actions from membership in a permitted trace
language.

The executable fixture additionally represents planning state.  After the
first effect, the safe replan introduces a review subgoal and then routes
traffic.  The drifted replan replaces the parent completion condition with
``production responds'' and omits review.  A per-action check accepts both
proposed operations.  Transition authorization rejects the replacement before
the route effect becomes reachable because the new success condition and
continuation traces do not refine the active parent.

\begin{figure*}[t]
\centering
\begin{tikzpicture}[
  x=1cm,y=1cm,
  box/.style={draw,rounded corners=2pt,align=center,minimum height=8mm,
    minimum width=27mm,inner sep=3pt,font=\footnotesize},
  safe/.style={box,draw=gtGreen,fill=gtGreenFill,text=gtGreen},
  bad/.style={box,draw=gtRed,fill=gtRedFill,text=gtRed},
  neutral/.style={box,draw=gtBlue,fill=gtBlueFill,text=gtBlue},
  every path/.style={-{Latex[length=2mm]},thick}
]
\node[neutral] (root) at (0,0) {approved root task\\diagnose deployment};
\node[neutral] (select) at (3.2,0) {select unreviewed\\artifact};
\node[safe] (review) at (6.6,1.15) {authorized refinement\\record independent review};
\node[safe] (safecommit) at (10.1,1.15) {route production\\allowed};
\node[bad] (skip) at (6.6,-1.15) {non-refining replan\\replace success; omit review};
\node[bad] (badcommit) at (10.1,-1.15) {route production\\must be blocked};
\draw[gtBlue] (root) -- (select);
\draw[gtGreen] (select) -- (review);
\draw[gtGreen] (review) -- (safecommit);
\draw[gtRed] (select) -- (skip);
\draw[gtRed,dashed] (skip) -- (badcommit);
\end{tikzpicture}
\caption{Matched traces with the same local tool affordances.  The unsafe
branch is rejected when its goal replacement is proposed, before the
production-routing effect can commit.}
\Description{Flow diagram from an approved deployment-diagnosis task to an
unreviewed-artifact selection.  The upper green branch adds an independent
review and then permits production routing.  The lower red branch replaces
the success condition, omits review, and is blocked before production
routing.}
\label{fig:counterexample}
\end{figure*}
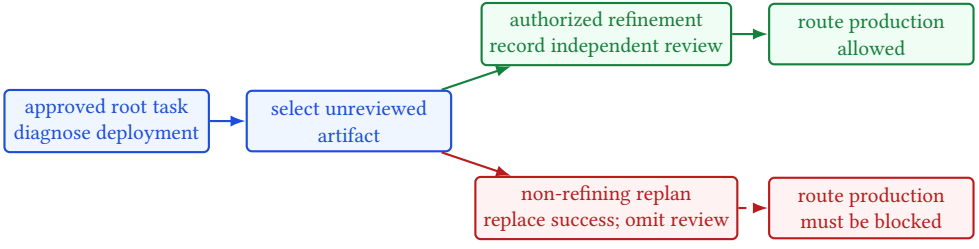

\subsection{System boundary}

The protected system has one trusted principal and one approved root task
contract.  The contract is expressed in a registered finite vocabulary of
state predicates, normalized effects, resources, obligations, and versions.
The agent may construct a directed acyclic goal graph below that root.  It may
create a child, replace an active goal during replanning, delegate a child to
another worker operating under the same root, join same-root sibling results,
cancel work, and close a goal.  Each operation is explicit and receipt-bearing.

Protected external effects pass through a commit adapter.  Read-only
observations can also be registered when their values affect authorization,
but the principal safety property concerns normalized external effects that
cross the adapter.  Effects outside this mediation boundary are outside the
execution represented by the model; the implementation therefore exposes a
fail-closed interface rather than claiming control over unregistered channels.

\subsection{Adversary}

The adversary may control the planner's proposed goal graph, free-form goal
text, retrieved content, intermediate memory, worker output, ordering of
requests, retries, and proposed effect arguments.  It may attempt to:

\begin{itemize}
  \item broaden a child's resources, effects, success predicate, or duration;
  \item split an unauthorized outcome across individually plausible siblings;
  \item replace a failed goal with a proxy that is easier to satisfy;
  \item reuse a witness after relevant state or policy changes;
  \item present another root's child as a compatible delegation or join;
  \item commit an effect before its transition is authorized, or replay a
  previously accepted request with different semantics.
\end{itemize}

The adversary cannot compromise the trusted registry, checker, state/version
source, receipt store and its integrity anchor, effect normalizer, or commit
adapter; rewrite the retained hash-linked receipt sequence without detection;
or bypass the registered effect boundary.  The reference artifact provides
tamper evidence under that trusted-storage assumption; it does not implement
public-key origin authentication.  Denial of service remains possible because
fail-closed authorization permits the checker to block execution when evidence
is incomplete.

\subsection{Security and functional properties}

The primary security property is prefix safety: every committed prefix and
every trace still admitted by the active goal's sound continuation abstraction
remain in the root contract's permitted trace language and avoid its modeled
forbidden states.  This property is independent of task completion.

The functional property is conditional: when every authorized child closes
under a valid refinement witness, active invariants hold at every registered
prefix, residual work and obligations are discharged, and the root closes,
the root's required terminal predicates hold.  The result does not require an
agent to find a plan or to terminate; it prevents successful closure from being
claimed through a non-refining structured goal graph.

\subsection{Scope}

The study covers single-root, single-principal authorization in finite
structured domains.  It directly covers create, replace, same-root delegate,
same-root join, invalidate, effect authorization, effect commit, cancel, and
close transitions.  Natural-language-to-policy compilation, ethical evaluation
of the principal's request, unmodeled physical consequences, compromised
trusted components, and combination of independent authorization roots are
separate problems.  These scope choices keep the stated theorem and the
executed implementation aligned.

\section{Background and Related Work}
\label{sec:related-work}

This work connects three established lines of research: reference-monitor and
trace-policy enforcement, hierarchical planning and formal refinement, and
safety evaluation for tool-using agents.  We use these foundations to answer a
specific systems question: how an enforcement service should decide whether
an untrusted agent may change the structured objective under which later
protected effects will execute.

\subsection{Authorization and trace enforcement}

Complete mediation and least privilege require security-relevant operations to
be checked at an authoritative boundary \cite{saltzer1975protection}.
Authorization logics make principals and delegated authority explicit
\cite{abadi1993accesscalculus}, while usage control accounts for decisions that
remain conditional on mutable attributes during use \cite{park2004ucon}.
History-based access control further establishes that the same current
operation can require different decisions after different prefixes
\cite{edjlali1998history}.  These results motivate the protected-effect
adapter and versioned decision context in this paper, but they do not define
authorization over changes to a planner's goal graph.

Trace enforcement supplies the closest formal foundation.  Security automata
characterize an important class of policies enforceable from execution
prefixes \cite{schneider2000enforceable}; edit automata broaden the available
interventions \cite{ligatti2005editautomata}; and later work sharpens the
boundary between monitorable and enforceable policies
\cite{basin2013enforceable}.  Runtime verification provides verdicts from
finite observations \cite{leucker2009runtimeverification,bauer2011ltlruntime}.
The safety/liveness distinction is therefore material: a finite bad prefix can
justify blocking, whereas eventual task progress is not guaranteed merely by
preventing bad prefixes \cite{alpern1985liveness}.  Our primary theorem is a
safety result under complete mediation and sound abstraction; the separate
success theorem is conditional on closure and does not assert termination.

Finite-trace temporal logic gives standard semantics for bounded executions
\cite{degiacomo2013ltlf}.  Temporal logic has also long constrained planning
and reactive control \cite{bacchus2000temporalsearch,kressgazit2009reactive},
and shielding can prevent an unsafe next action, including under partial
observability \cite{alshiekh2018shielding,carr2023partialshielding}.  Thus,
neither temporal policies nor online intervention is a novelty claim here.
The additional object checked in this paper is a proposed change to the
authorized goal structure and its modeled continuations, before that change
can make new protected effects reachable.

\subsection{Hierarchical planning and refinement}

Hierarchical task network (HTN) planning recursively decomposes abstract tasks
according to domain methods.  Its formal complexity, representative planning
systems, and design space are well established
\cite{erol1996htncomplexity,nau2003shop2,georgievski2015htnoverview,bercher2019hierarchicalsurvey}.
Hierarchical goal network (HGN) planning instead gives decomposed goals
state-based semantics and explicit method-correctness conditions
\cite{shivashankar2012hgn}; recent work extends HGN reasoning to stochastic
outcomes \cite{chan2026probabilistichgn}.  Refinement mappings likewise predate
the present setting \cite{abadi1991refinement}.

These formalisms explain how trusted domain knowledge can structure planning.
They do not, by themselves, answer whether a decomposition proposed at runtime
by an untrusted agent is authorized under a security mandate, whether evidence
for that decision remains current after a policy or relevant-state change, or
whether a protected effect still matches the active authorized goal at commit
time.  This paper therefore does not introduce a planning formalism.  It uses
a registered finite transition model and ordinary refinement obligations as
inputs to an authorization decision over dynamic graph mutations.

\subsection{Tool-using-agent security and long-horizon evaluation}

Peer-reviewed environments now support substantially stronger evaluation than
author-written demonstrations alone.  ToolEmu emulates high-stakes tool
interfaces for scalable risk discovery \cite{ruan2024toolemu}; AgentDojo
jointly evaluates utility and prompt-injection security in a stateful tool
environment \cite{debenedetti2024agentdojo}; InjecAgent and Agent Security
Bench broaden indirect-injection and attack/defense evaluation
\cite{zhan2024injecagent,zhang2025agentsecuritybench}.  SafeArena supplies
paired benign and harmful web tasks \cite{tur2025safearena}.  AppWorld,
WebArena, OSWorld, and $\tau$-bench provide multi-step app, web, computer-use,
and tool--user interaction tasks with external state or task-success checks
\cite{trivedi2024appworld,zhou2024webarena,xie2024osworld,yao2025taubench}.
These sources motivate the independent task origins and state-based oracles in
our evaluation; success on one environment is not treated as evidence of
universal safety.

Long-horizon evidence is now available in formal proceedings as well as
preprints.  AgentLAB evaluates adaptive attacks including objective drift,
tool chaining, task injection, and memory poisoning
\cite{jiang2026agentlab}.  MT-AgentRisk evaluates harmful multi-turn tool use
and a tool-oriented defense \cite{li2026multiturn}.  Goal
misgeneralization and uncertainty about the intended reward provide related
evidence that optimizing a learned or inferred objective can diverge from the
designer's intended result
\cite{langosco2022goalmisgeneralization,hadfieldmenell2017inversereward}.
These works establish the practical importance of sequence-level evaluation.
They do not make every long execution unsafe, and their threat categories are
prior work rather than contributions of this paper.

Recent preprints add trajectory-focused evidence.  ATBench and HINTBench study
trajectory-level and intrinsic non-attack failures
\cite{li2026atbench,wang2026hintbench}.  SafetyDrift predicts violation risk
from accumulated trajectory state \cite{dhodapkar2026safetydrift}, while the
Constraint Drift position paper argues that constraints must remain operative
across memory, delegation, tools, and audit boundaries
\cite{li2026constraintdrift}.  We use these preprints to delimit current
questions and potential evaluation cases, not as substitutes for the
peer-reviewed foundations above.

\subsection{Runtime controls for language-model agents}

Several recent systems are direct technical neighbors.  AgentSpec, published
at ICSE 2026, provides a language and runtime for customizable agent
constraints \cite{wang2026agentspec}.  Formal Methods Meet LLMs, published at
FAccT 2026, applies temporal-logic auditing, monitoring, prediction, and
intervention to advanced AI systems \cite{alamdari2026formalmethods}.  CaMeL,
published at SaTML 2026, separates trusted control flow from untrusted data and
uses capabilities to constrain tool-mediated information flow
\cite{debenedetti2026camel}.  These peer-reviewed systems rule out claims based
only on a policy language, temporal monitoring, runtime intervention, or
capability-gated tool calls.

CaMeLoT translates an agent-generated plan to a finite-state transition system
and checks temporal policies before execution \cite{nikolaou2026camelot}.  Its
arXiv record reports acceptance at CAMLIS 2026, but a proceedings record was
not available at the 2026-10-02 literature cutoff; we therefore label it
\emph{accepted}, not formally published.  It is the closest plan-verification
comparison because it reasons about the whole proposed plan rather than an
isolated tool call.  Our evaluation therefore distinguishes initial
plan verification from authorization of graph changes introduced after
execution has begun, including invalidation and reauthorization after relevant
version changes.

The remaining closest works were preprints at the cutoff.  Agent-C constrains
tool-call generation using temporal specifications and SMT solving
\cite{kamath2025agentc}; VIGIL checks order, parameter relations, and value
flow across agent--tool events \cite{li2026vigil}.  RGE monitors role, goal,
and evidence drift at trajectory prefixes \cite{he2026rge}.  Agent Behavioral
Contracts covers preconditions, invariants, governance policies, and recovery
\cite{bhardwaj2026behavioralcontracts}, while Structural Enforcement of Goal
Integrity separates policy, authorization, and execution and tracks lineage
\cite{xiang2026goalintegrity}.  Their publication status limits the weight of
their reported empirical results, but their mechanisms still constrain our
novelty language: temporal enforcement, drift monitoring, contracts,
architectural separation, and lineage are not claimed as new.

\begin{table*}[t]
  \centering
  \caption{Positioning by controlled object and enforcement point.  Status is
  recorded at the 2026-10-02 cutoff: PR = peer-reviewed publication, A =
  accepted with proceedings pending, and PP = preprint.  ``Not explicit''
  means that the cited work does not specify the listed mechanism as part of
  its reported design; it is not an assertion that the capability is
  impossible or absent from every implementation.}
  \label{tab:related-work-comparison}
  \scriptsize
  \renewcommand{\arraystretch}{0.92}
  \setlength{\tabcolsep}{3pt}
  \begin{tabularx}{\textwidth}{@{}L{0.17\textwidth}L{0.19\textwidth}L{0.15\textwidth}YYY@{}}
    \toprule
    Work and status & Primary controlled object & Decision point & Dynamic goal-graph mutation & Version-bound reauthorization & Protected-effect gate \\
    \midrule
    Security/edit automata \cite{schneider2000enforceable,ligatti2005editautomata} (PR)
      & Event traces
      & During execution
      & Not a planning model
      & Policy-state dependent, but no goal witness
      & Abstract enforcement event \\
    HTN/HGN planning \cite{nau2003shop2,shivashankar2012hgn,chan2026probabilistichgn} (PR)
      & Task or goal decomposition
      & Plan construction/progression
      & Decomposition, without a security-authorization adversary
      & Not an authorization lifecycle
      & Not its control boundary \\
    AgentSpec \cite{wang2026agentspec} (PR)
      & Structured behavioral rules
      & Runtime events
      & Not explicit
      & Not explicit
      & Runtime enforcement \\
    Formal Methods Meet LLMs \cite{alamdari2026formalmethods} (PR)
      & Temporally extended constraints
      & Audit, monitor, predict, intervene
      & Not explicit
      & Monitor state evolves; witness invalidation not explicit
      & Intervention on predicted or observed behavior \\
    CaMeL \cite{debenedetti2026camel} (PR)
      & Control/data flow and capabilities
      & Program/tool execution
      & Not explicit
      & Not explicit
      & Capability-mediated tool calls \\
    CaMeLoT \cite{nikolaou2026camelot} (A)
      & Generated finite-state plan
      & Before plan execution
      & Reverification of mid-run mutations not explicit
      & Not explicit
      & Complements CaMeL execution controls \\
    Agent-C \cite{kamath2025agentc} (PP)
      & Temporal tool-use constraints
      & During tool-call generation
      & Not explicit
      & Not explicit
      & Constrained generation blocks noncompliant calls \\
    VIGIL \cite{li2026vigil} (PP)
      & Cross-call order, arguments, and value flow
      & Agent--tool runtime events
      & Not explicit
      & Not explicit
      & Runtime enforcement \\
    RGE \cite{he2026rge} (PP)
      & Role, goal, and evidence drift
      & Online prefix monitoring
      & Observes evolving trajectories
      & Deterministic monitor updates, not a refinement witness
      & Monitor or intervention decision \\
    Behavioral contracts and structural goal integrity
      \cite{bhardwaj2026behavioralcontracts,xiang2026goalintegrity} (PP)
      & Contracts, invariants, intent, and lineage
      & Architecture and runtime gates
      & Goal evolution is addressed at an architectural or contract level
      & Version-bound graph-witness invalidation not explicit
      & Runtime and output gates \\
    \textbf{This work}
      & Same-root structured goal graph and residual effect traces
      & Every graph mutation and protected effect
      & Create, replace, delegate, cancel, and same-root join
      & Relevant policy/state changes invalidate retained witnesses
      & Active goal and residual trace checked before commit \\
    \bottomrule
  \end{tabularx}
\end{table*}

\subsection{Residual contribution boundary}

Table~\ref{tab:related-work-comparison} compares reported control objects and
decision points rather than ranking systems.  Building on reference monitoring,
trace enforcement, temporal specifications, refinement, hierarchical
decomposition, static plan verification, drift monitoring, contracts, and
tool-call gates, this paper contributes their authorization connection in the
stated model:
an untrusted agent proposes a mutation to a versioned, single-root structured
goal graph; a separate checker requires machine-checkable evidence that the
new continuation conservatively refines the active authorization; relevant
policy or state changes invalidate that evidence; and a commit adapter checks
each protected effect against the active authorized node and residual trace.

This boundary excludes arbitrary natural-language entailment, independent
authorization roots, and unmodeled real-world consequences.  The empirical
evaluation tests whether the complete combination covers typed graph changes
that local controls, final-state detection, and an exact continuation-trace
comparator do not jointly capture, while retaining valid replanning,
delegation, and recovery under matched policies and tasks.  Learned drift
monitors and HTN/HGN planners address different decision problems and are
therefore positioned conceptually rather than reimplemented as accuracy
baselines.

\section{Structured Goal-Transition Semantics}
\label{sec:formalization}

We formalize goal creation, replacement, same-root delegation, and same-root
composition as authorization events.  The model is intentionally restricted to
a trusted structured task domain.  Free-form text may motivate a proposal, but
neither a model explanation nor similarity between natural-language goals is an
authorization fact.

\subsection{Transition system and root task contract}

Let $\Effects$ be the alphabet of normalized protected effects and
$\mathcal S$ the set of modeled external states.  A registered transition
relation $s\xrightarrow{e}s'$ describes the possible states after effect
$e\in\Effects$.  Its trace lifting is written
$\mathsf{Post}(s,u)\subseteq\mathcal S$ for $u\in\Effects^*$.  It is defined
recursively by $\mathsf{Post}(s,\epsilon)=\{s\}$ and
\(
\mathsf{Post}(s,eu)=
\bigcup_{x\in\mathsf{Post}(s,e)}\mathsf{Post}(x,u)
\).
The relation may be nondeterministic.  We write $v\preceq u$ when $v$ is a
prefix of $u$, $\epsilon$ for the empty trace, and
$t^{-1}L=\{u\mid tu\in L\}$ for the residual of a language $L$ after $t$.

\begin{definition}[Root task contract]
\label{def:root-contract}
A root task contract is
\begin{equation}
 C=(\rho,\eta,s_0,L_C,B_C,Q_C,R_C,O_C),
\label{eq:root-contract}
\end{equation}
where $\rho$ is the authorization-root identifier, $\eta$ is its immutable
contract epoch, $s_0$ is the initial modeled state, $L_C\subseteq\Effects^*$
is a prefix-closed permitted trace language containing $\epsilon$,
$B_C\subseteq\mathcal S$ is the forbidden-state set, $Q_C$ is the required
root-closing predicate, $R_C$ is a resource envelope, and $O_C$ is the set of
unresolved obligations.  Resource envelopes belong to a registered preorder
$(\mathcal R,\sqsubseteq_R)$ with a sound domain-specific aggregation
operator.  The contract is supplied or explicitly approved by the principal;
deriving it from unrestricted natural language is outside the model.
\end{definition}

The policy version, relevant-state version, and goal-graph revision are
recorded in
\begin{equation}
 \nu=(\eta,\nu_P,\nu_S,\kappa).
\label{eq:version-vector}
\end{equation}
Here $\nu_S$ changes whenever an out-of-band observation changes a state
component on which a proof depends.  A mediated protected effect instead
updates the modeled state, committed-trace digest, graph revision, residual
languages, and selected witness atomically; it therefore cannot leave the old
evidence current even without a second state-version increment.  A new root
task contract creates a new epoch; it is not an ordinary goal-graph edit.  Each policy version names a
structured admissibility predicate
$\mathsf{Policy}_{\nu_P}(s,t,K,R,O)$ over the proposed future-trace set,
resources, and obligations.  A policy change need not be monotone: freshness,
rather than comparison with the old policy, is the authorization rule.

For a current state $s$, committed trace $t$, and prefix-closed set $K$ of
permitted future traces, define
\begin{equation}
\begin{split}
 \mathsf{SafeFuture}_C(s,t,K)\ \Longleftrightarrow\ {}&t\in L_C\ \land\\
 &\forall u\in K\;\forall v\preceq u:\quad
 tv\in L_C\ \land\
 \mathsf{Post}(s,v)\cap B_C=\varnothing.
\end{split}
\label{eq:safe-future}
\end{equation}
This predicate constrains order, multiplicity, and state-dependent composition;
it is not a flat set-membership test on individual effects.

\subsection{Versioned goal graph and same-root composition}

A goal graph is a finite acyclic graph with containment edges from a parent to
its refinements and precedence edges among siblings.  Each node is
\begin{equation}
 g=(i,\rho_g,p_g,k_g,Q_g,K_g,D_g,R_g,O_g,F_g,\zeta_g).
\label{eq:goal-node}
\end{equation}
$i$ is a unique node identifier; $\rho_g$ is its authorization root; $p_g$ is
its unique containment parent except at the root; $k_g$ is its node revision;
$Q_g$ is its structured closing predicate; $K_g\subseteq\Effects^*$ is a
prefix-closed, sound over-approximation of all modeled finite traces the node
may produce; $D_g\subseteq K_g$ is the subset on which the node may close;
$R_g$, $O_g$, and $F_g$ are its resource envelope, unresolved obligations,
and prefix-wise frame predicates; and $\zeta_g$ is the version vector against
which the node was authorized.  Text descriptions are audit metadata only.

Consider children $c_1,\ldots,c_m$, a residual completion component $d$, and
a witnessed partial order $\prec$ over those components.  Let
\begin{equation}
 (K_\Delta,D_\Delta)=
 \mathsf{Comp}_{\prec}
 \bigl((K_{c_1},D_{c_1}),\ldots,(K_{c_m},D_{c_m}),(K_d,D_d)\bigr)
\label{eq:composition}
\end{equation}
denote their ordinary trace composition: $K_\Delta$ contains every finite
prefix generated by a scheduler that preserves each component's internal
order and $\prec$; $D_\Delta$ contains exactly the traces on which all required
children and the residual component close.  Thus incomparable siblings are
checked under every registered interleaving.  In particular,
$\mathsf{Comp}_{\prec}$ is not set union: it retains order, multiplicity, and
joint state reachability.

Every accepted edit increments the graph revision $\kappa$.  Replacement
creates a new node revision rather than changing an authorized node in place.
Delegation changes the executor but not $\rho_g$.  This paper has one
principal and one root per execution; a cross-root join is outside the model
and is rejected.

Ordinary sibling nodes are not implicitly concurrent.  At a commit boundary,
a fresh witness selects one leaf-to-root path
\(\mathsf{Path}(w)=(g_0,\ldots,g)\).  The continuations that can authorize the
next effect are
\begin{equation}
 K_{X,w}=\bigcap_{h\in\mathsf{Path}(w)} K_h,
 \label{eq:witness-path-language}
\end{equation}
where every $K_h$ is residualized at the same committed prefix.  A graph edit
advances $\kappa$, so witnesses for other sibling paths become stale until
explicitly reauthorized.  Multiple siblings enter one current composition only
through an explicit join, whose language is the complete
$\mathsf{Comp}_{\prec}$ construction in Eq.~\eqref{eq:composition}.  This rule
prevents an informal union of individually authorized siblings from becoming
an execution policy.

\subsection{Transition witnesses}

A proposed edit $\Delta$ below active parent $p$ carries a transition witness
\begin{equation}
 w=\langle\rho,\eta,\kappa_{\rm base},\kappa_{\rm new},
       \nu_P,\nu_S,i_p,k_p,h(\Delta),\Pi\rangle,
\label{eq:witness}
\end{equation}
where $\Pi$ contains the structured proof evidence.  The checker returns
$\Allow$ only if all of the following obligations hold.

\begin{enumerate}[label=\textbf{W\arabic*.},ref=W\arabic*]
  \item \label{wit:binding}
  The witness binds the current root and epoch, exact parent revision, current
  base graph revision, proposed new revision, and canonical edit.

  \item \label{wit:root}
  Every new node and delegation remains below root $\rho$; containment remains
  acyclic and the referenced parent is active.

  \item \label{wit:trace}
  The edit supplies a complete child--residual composition
  $(K_\Delta,D_\Delta)$, and its future language is non-expansive:
  \begin{equation}
    K_\Delta\subseteq K_p\cap t^{-1}L_C.
    \label{eq:trace-nonexpansion}
  \end{equation}
  Moreover, $K_\Delta$ soundly contains all modeled concrete continuations of
  the selected edited path, its completion component is bound to the current
  parent residual, and the current authorization policy satisfies
  $\mathsf{Policy}_{\nu_P}(s,t,K_\Delta,R_\Delta,O_\Delta)$.

  \item \label{wit:state}
  Every prefix admitted by $K_\Delta$ preserves the root forbidden-state
  condition and every frame active on the selected path:
  \begin{equation}
  \begin{split}
   \forall u\in K_\Delta\;\forall v\preceq u:\quad
   &\mathsf{Post}(s,v)\cap B_C=\varnothing\ \land\\
   &\forall F\in\mathsf{ActiveFrames}(\Delta,v)\;
     \forall s'\in\mathsf{Post}(s,v):F(s').
  \end{split}
  \label{eq:frame-preservation}
  \end{equation}

  \item \label{wit:resources}
  The registered aggregate of simultaneously exercisable child resources is
  bounded by both $R_p$ and $R_C$.  A child cannot introduce an authorization
  obligation: its unresolved set equals $O_p$ minus obligations accompanied by
  independently checkable discharge evidence from the contract's trusted
  catalog.  Merely adding or deleting an obligation is invalid.  Explicit joins
  apply this rule to the complete affected sibling composition.

  \item \label{wit:success}
  For every $u\in D_\Delta$ and every $s'\in\mathsf{Post}(s,u)$, validated
  closure of the children and residual component, together with their
  discharged obligations, entails $Q_p(s')$.  At the root,
  $Q_{g_0}(s')\Rightarrow Q_C(s')$.  The relation is evaluated in the common
  parent-closing state, not at unrelated earlier states.

  \item \label{wit:freshness}
  The contract epoch, policy version, relevant-state version, parent revision,
  and graph base revision equal their trusted current values at authorization.
\end{enumerate}

These are semantic obligations over a registered finite-state or otherwise
decidable task theory; they are not inferred from the checker's own verdict.
When exact checking is unavailable, a sound over-approximation may prove all
universal obligations and return $\Allow$.  A concrete counterexample or an
exact structural mismatch returns $\Deny$.  Spurious unsafe behavior in an
over-approximation returns $\Indeterminate$, and a changed version returns
$\Reauthorize$.  Both latter outcomes block protected effects.

\subsection{Operational rules and invalidation}

An execution configuration is
\begin{equation}
 X=(C,s,t,G,\nu,W,q),
\label{eq:configuration}
\end{equation}
where $W$ is the set of accepted witnesses and
$q\in\{\mathsf{runnable},\mathsf{suspended},\mathsf{closed}\}$.  The active
goal graph and a fresh selected witness determine the prefix-closed future set
$K_{X,w}$ in Eq.~\eqref{eq:witness-path-language}.  The transition relation has
five classes:

\begin{description}[leftmargin=1.2em,style=nextline]
  \item[Authorize a graph edit.]
  From a runnable configuration, validate W1--W7 against the current state and
  base revision; atomically install the new graph revision and witness.  A
  failed or indeterminate check does not alter the active graph.

  \item[Commit a protected effect.]
  Normalize the proposed effect $e$, require a current witness on the selected
  path and $e^{-1}K_{X,w}\ne\varnothing$, select an actual modeled successor
  $s'\in\mathsf{Post}(s,e)$, and atomically update
  \begin{equation}
    (s,t,K_{X,w})\leftarrow(s',te,e^{-1}K_{X,w}).
    \label{eq:effect-residual}
  \end{equation}
  No other path may commit a protected effect.

  \item[Invalidate.]
  An out-of-band change to the authorization policy or a relevant observed
  state component advances the corresponding version before the next protected
  commit, invalidates every dependent witness, and sets
  $q=\mathsf{suspended}$.  Goal-graph mutation has no out-of-band interface:
  an accepted graph edit instead atomically advances $\kappa$, stales all old
  witnesses, installs the selected edit's new witness, and keeps the execution
  runnable.

  \item[Reauthorize or cancel.]
  Reauthorization recomputes the complete selected path against current
  versions.  If no active non-root goal exists, a suspended root may instead
  resume by atomically installing a new graph edit that passes W1--W7 against
  the current versions.  These are the only ways to return a suspended
  execution to $\mathsf{runnable}$.  Cancellation removes future effects but
  creates no authority and does not by itself resume a suspended execution.
  A later recovery may submit a full current-version replacement that
  explicitly targets the cancelled node; acceptance changes the cancelled
  tombstone to replaced and installs a new active node under W1--W7.  Neither
  operation reuses an old witness, and cancellation alone is not evidence of
  successful closure.

  \item[Close.]
  A non-root goal closes only under a current witness, after its actually
  executed suffix belongs to the current $D_g$, its structured predicate and
  discharge evidence validate in the common closing state, and every required
  descendant in the witnessed composition has closed.  The root closes only
  from a runnable configuration after every required descendant has resolved,
  the committed trace is a permitted completion, all root obligations are
  discharged, and $Q_C$ holds.  Closing a non-root goal only changes that
  goal's status, advances $\kappa$, and stales the prior witnesses while $q$
  remains runnable; another active path must be reauthorized before a protected
  commit.  Only root closure sets $q=\mathsf{closed}$.  Closure commits no
  external effect by itself.
\end{description}

Replacing the root contract requires principal authorization of a new epoch
and begins a new execution; it is not an internal step of the transition system
proved here.  A new execution may import a historical prefix only if its new
contract explicitly admits that prefix.  Ordinary subgoal replacement cannot
change the root contract.

\subsection{Assumptions and claim boundary}

The results use the following explicit premises.

\begin{enumerate}[label=\textbf{A\arabic*.},ref=A\arabic*]
  \item \label{asm:mediation} Every protected external effect is completely
  mediated by Eq.~\eqref{eq:effect-residual}.
  \item \label{asm:fidelity} Effect normalization and the commit adapter are
  faithful: the checked effect is the effect executed.
  \item \label{asm:model} The registered transition relation and state
  observations soundly cover the modeled effects and nondeterminism used by the
  contract.
  \item \label{asm:future} Each $K_g$ soundly over-approximates the modeled
  finite continuations of its node.
  \item \label{asm:checker} The checker soundly validates W1--W7 in the trusted
  structured theory.
  \item \label{asm:atomicity} Version advance, witness invalidation, graph
  installation, and effect commit have a common atomic order.
  \item \label{asm:initial} The initial root node is well formed,
  $\mathsf{SafeFuture}_C(s_0,\epsilon,K_{g_0})$ holds, and its resource and
  obligation fields equal those of $C$; its closing predicate satisfies
  $Q_{g_0}\Rightarrow Q_C$.
  \item \label{asm:single-root} One principal and one authorization root govern
  an execution; cross-root and multi-principal composition are rejected.
\end{enumerate}

The system does not guarantee that the principal chose a desirable contract,
discover unmodeled physical consequences, prove termination, or decide
arbitrary natural-language entailment.  An exogenous event may independently
create a forbidden world state; A3 and A6 require it to invalidate dependent
witnesses before another protected effect, but authorization cannot undo that
external event.

\subsection{Results}

\begin{theorem}[B1: root-task safety]
\label{thm:root-safety}
Under A1--A8, every execution generated by the operational rules preserves
$t\in L_C$.  Every protected effect committed from a runnable configuration
uses a fresh same-root witness and produces no modeled successor in $B_C$.
Furthermore, whenever a runnable post-state has a current selected witness,
that path's future satisfies $\mathsf{SafeFuture}_C$; installed goals remain within the root and parent
resource envelopes and cannot discard an unresolved obligation.  Thus no
modeled protected effect expands the root trace language, bypasses the policy
version under which it is admitted, or causes a root-forbidden state.
\end{theorem}

Theorem~\ref{thm:root-safety} is a conditional invariant-preservation result.
In particular, A3--A6 carry the fidelity, soundness, checker, and atomicity
premises; the proof does not by itself establish that a concrete implementation
satisfies them.  Section~\ref{sec:results} reports separate conformance evidence
for the finite reference profile.

\begin{theorem}[B2: conditional root-success preservation]
\label{thm:success}
Assume A1--A8, validated leaf closure, and discharge of every obligation
required by the accepted witnesses.  If the finite acyclic goal graph closes
at its root, then its terminal modeled state satisfies $Q_C$.
\end{theorem}

Theorem~\ref{thm:success} is conditional functional correctness.  It does not
assert that the agent terminates, finds a closing plan, or avoids
$\Indeterminate$ decisions.

\begin{lemma}[B3: insufficiency of a memoryless local allowlist]
\label{lem:local-insufficiency}
There is a prefix-closed trace policy over individually permitted effects for
which no control depending only on the current effect and a static flat
allowlist is both safe and permissive for every safe singleton trace.
\end{lemma}

This construction instantiates the established history-dependence boundary for
self-generated subgoals: the stated memoryless decision class must retain trace
or equivalent goal-composition state to preserve both properties.

\begin{algorithm}[t]
\caption{Exact transition authorization in a registered finite domain}
\label{alg:exact-authorization}
\begin{algorithmic}[1]
\Require Contract $C$, current $(s,t,G,\nu)$, proposed edit $\Delta$, witness $w$
\If{the root lineage, graph structure, or canonical edit binding is invalid}
  \State \Return $\Deny$
\EndIf
\If{the epoch, policy, state, parent, or graph version is stale}
  \State \Return $\Reauthorize$
\EndIf
\State Enumerate every registered topological schedule of $\Delta$ and its residual work
\State Construct the exact finite $(K_\Delta,D_\Delta)$ and every reachable prefix state
\If{trace inclusion, policy, forbidden-state, or active-frame checking fails}
  \State \Return $\Deny$
\EndIf
\If{resource aggregation, obligation preservation/discharge, or same-root lineage fails}
  \State \Return $\Deny$
\EndIf
\If{a closing state in $D_\Delta$ fails parent-success refinement}
  \State \Return $\Deny$
\EndIf
\State \Return $\Allow$
\end{algorithmic}
\end{algorithm}

\begin{proposition}[B4: exact finite-domain decision property]
\label{thm:relative-completeness}
Suppose $\mathcal S$, $\Effects$, the proposed graph, its composed trace
language, and every registered constraint domain are finite, and the checker
implements Algorithm~\ref{alg:exact-authorization}.  The algorithm terminates
and returns \textnormal{$\Allow$} if and only if a current proposed edit
satisfies W1--W7.  In particular, an admissible edit is not rejected or marked
\textnormal{$\Indeterminate$} in this exact profile.
\end{proposition}

\begin{proposition}[B5: universal-safety monotonicity under a more precise sound abstraction]
\label{prop:precision}
Let $T$ be the concrete continuation language and let prefix-closed sound
abstractions $A_1,A_2$ satisfy $T\subseteq A_2\subseteq A_1$ under the same
contract and version, with identical non-language obligations already
discharged.  If the universal $\mathsf{SafeFuture}_C$ subcheck succeeds on
$A_1$, it also succeeds on $A_2$; either success is safe for $T$.  A concrete
counterexample-backed \textnormal{$\Deny$}
supported by a concrete trace remains valid.  With fixed concrete-counterexample
evidence, abstraction refinement alone may resolve an inconclusive universal
safety subcheck to success, but cannot manufacture a new concrete denial.  This
proposition does not assert monotonicity of arbitrary policy predicates that
depend non-monotonically on the representation of $K$.
\end{proposition}

Appendix~\ref{app:formal-proofs} proves all five statements.  The executable
finite model is corroborating evidence, not a substitute for those proofs: it
exhausts only selected obligations within the declared abstraction.  Its
independent oracle recomputes the 13 initial proposal decisions and trace sets;
it does not independently check the complete 340-state transition system,
freshness lifecycle, or the reference implementation.

\section{Runtime Authorization Protocol}
\label{sec:protocol}

The protocol places an authorization boundary between goal-graph mutation and
effect execution.  It does not prescribe how a planner generates goals.  It
requires the planner to expose each security-relevant mutation in a typed form
that can be checked against the current root task contract.

Figure~\ref{fig:protocol-architecture} separates the two decisions.  A graph
mutation must be authorized before it becomes active; a later protected call
must still match the current witness and versions at the commit boundary.
Neither a planner-produced explanation nor an earlier allow decision bypasses
the second check.

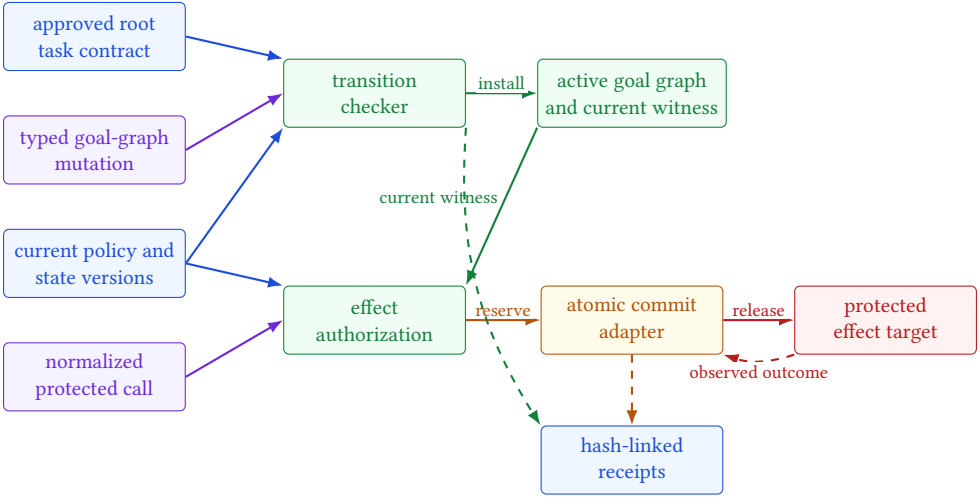
\begin{figure*}[t]
\centering
\begin{tikzpicture}[
  x=1cm,y=1cm,
  box/.style={draw,rounded corners=2pt,align=center,minimum height=9mm,
    minimum width=24mm,inner sep=3pt,font=\footnotesize},
  trusted/.style={box,draw=gtBlue,fill=gtBlueFill,text=gtBlue},
  untrusted/.style={box,draw=gtPurple,fill=gtPurpleFill,text=gtPurple},
  decision/.style={box,draw=gtGreen,fill=gtGreenFill,text=gtGreen},
  boundary/.style={box,draw=gtAmber,fill=gtAmberFill,text=gtAmber},
  target/.style={box,draw=gtRed,fill=gtRedFill,text=gtRed},
  arrow/.style={-{Latex[length=2mm]},thick},
  evidence/.style={arrow,dashed}
]
\node[trusted]   (contract) at (0,2.25) {approved root\\task contract};
\node[untrusted] (mutation) at (0,0.75) {typed goal-graph\\mutation};
\node[trusted]   (versions) at (0,-0.75) {current policy and\\state versions};
\node[untrusted] (call)     at (0,-2.25) {normalized\\protected call};

\node[decision]  (transition) at (3.7,1.50) {transition\\checker};
\node[decision]  (effect)     at (3.7,-1.50) {effect\\authorization};
\node[decision]  (graph)      at (7.1,1.50) {active goal graph\\and current witness};
\node[boundary]  (commit)     at (7.1,-1.50) {atomic commit\\adapter};
\node[target]    (target)     at (10.45,-1.50) {protected\\effect target};
\node[trusted]   (receipts)   at (7.1,-3.35) {hash-linked\\receipts};

\draw[arrow,gtBlue] (contract.east) -- (transition.north west);
\draw[arrow,gtPurple] (mutation.east) -- (transition.west);
\draw[arrow,gtBlue] (versions.east) -- (transition.south west);
\draw[arrow,gtGreen] (transition.east) --
  node[above,font=\scriptsize,fill=white,inner sep=1pt]{install} (graph.west);

\draw[arrow,gtPurple] (call.east) -- (effect.west);
\draw[arrow,gtBlue] (versions.east) -- (effect.north west);
\draw[arrow,gtGreen] (graph.south west) --
  node[above left,font=\scriptsize,fill=white,inner sep=1pt]{current witness}
  (effect.north east);
\draw[arrow,gtAmber] (effect.east) --
  node[above,font=\scriptsize,fill=white,inner sep=1pt]{reserve} (commit.west);
\draw[arrow,gtRed] (commit.east) --
  node[above,font=\scriptsize,fill=white,inner sep=1pt]{release} (target.west);
\draw[evidence,gtRed] (target.south west) to[bend left=18]
  node[below,font=\scriptsize,fill=white,inner sep=1pt]{observed outcome}
  (commit.south east);

\draw[evidence,gtGreen] (transition.south east) to[bend right=18]
  (receipts.north west);
\draw[evidence,gtAmber] (commit.south) -- (receipts.north);
\end{tikzpicture}
\caption{Two-stage authorization boundary.  Goal-graph mutations are checked
before activation; protected effects are reserved and committed only against
the current installed witness.  Solid arrows carry decisions or effects, and
dashed arrows retain audit evidence or observed outcomes.}
\Description{Architecture diagram with trusted root-contract and version
inputs, untrusted typed mutation and protected-call inputs, a transition
checker that installs a current witness, and a separate effect-authorization
and atomic-commit path to the protected target.  Dashed arrows connect both
stages to hash-linked receipts and return the observed effect outcome.}
\label{fig:protocol-architecture}
\end{figure*}

\subsection{Objects and identifiers}

Every root task receives a unique \mcode{rootId}, an immutable contract digest,
and monotonically increasing policy and state versions.  Each goal has a unique
\mcode{goalId}, exactly one parent except for the root, a root identifier, a
goal-local version, and a lifecycle state
\(\{\mcode{active},\mcode{closed},\mcode{cancelled},\mcode{replaced},
\mcode{joined}\}\).  A proposal is a request, not a persisted goal state.  The
root separately moves among \mcode{active}, \mcode{suspended}, and
\mcode{closed}; a version change suspends the root and stales witnesses rather
than rewriting each goal as ``invalidated.''  A transition witness binds all fields used by the
decision: parent and proposed goal digests, current committed trace digest,
state and policy versions, the checker profile, and the decision.

An effect request binds an idempotency key, active goal identifier, root
identifier, normalized effect, expected trace digest, and the witness that made
the effect reachable.  A receipt binds the request to one outcome and, for a
successful commit, to the new trace and state digests.  Reusing an idempotency
key with a different request body is rejected.

\begin{table*}[t]
\caption{Registered protocol operations and their authorization obligations.}
\label{tab:operations}
\small
\begin{tabularx}{\textwidth}{L{0.20\textwidth}L{0.30\textwidth}Y}
\toprule
Operation & State change & Required checks \\
\midrule
\mcode{registerRoot} & create root and initial active goal & principal approval, schema validity, nonempty permitted language, initial invariant \\
\mcode{proposeGoal} & proposed child becomes active & same root; parent active; trace, resource, obligation, frame, success, and version refinement \\
\mcode{replaceGoal} & invalidate old goal; activate replacement & all child checks plus preservation of already committed prefix and residual work \\
\mcode{delegateGoal} & assign an active child to a worker & same root, attenuated resource scope, bound worker identity, no new authorization root \\
\mcode{joinGoals} & authorize sibling composition & same root; witnessed order/interleaving; joint trace, multiplicity, frame, and completion checks \\
\mcode{authorizeEffect} & reserve next normalized effect & active current goal; current versions; effect admitted by residual language and state guard \\
\mcode{commitEffect} & append effect and update state/trace & exact reservation; idempotency; faithful adapter; atomic trace and residual update \\
\mcode{changeVersion} & advance policy/state version; suspend root & typed mutation; atomic version/lifecycle update; old witnesses and reservations become stale \\
\mcode{reauthorize} & issue current witness; resume root & complete selected-path recheck under current state, policy, parent, and graph revisions \\
\mcode{cancelGoal} & remove a non-root future path & active goal, resolved descendants, no authority or effect creation \\
\mcode{closeGoal} & validate goal completion & success predicate, inherited obligations, active frames, and residual completion all satisfied \\
\mcode{closeRoot} & make root terminal & every required descendant resolved; permitted completed trace; root predicate and obligations satisfied \\
\bottomrule
\end{tabularx}
\end{table*}

\subsection{Transition decision}

For a proposed graph mutation, the checker first performs structural checks:
identifiers are unique, the graph remains acyclic, every non-root node has one
parent, and all composed nodes share the same root.  It then checks the
registered versions and rejects a stale proposal with \Reauthorize{}.  Typed
resource and obligation checks are followed by reachability over the proposed
finite continuation system.  A concrete counterexample to the root policy or
parent refinement returns \Deny{}.  If a sound over-approximation contains an
unsafe abstract path but cannot establish that the path is concrete, the
checker returns \Indeterminate{}.  The implementation never converts either
outcome to \Allow{}.

An ordinary edit is checked over its selected child path, every ancestor
residual, and the remaining completion that binds it to its parent.  A graph
revision makes other sibling witnesses stale.  If effects from several
siblings are to become jointly executable, the planner must request an
explicit join; that request records the sibling set and partial order and
checks every registered interleaving.  Later graph mutations produce new
witnesses; they do not edit a previous decision in place.

\subsection{Effect reservation and commit}

Effect execution uses a two-step interface.  \mcode{authorizeEffect} verifies
the current goal and versions, normalizes the requested operation, advances a
candidate residual state, and creates a single-use, context-bound reservation.  The adapter
may commit only the exact reserved normalized effect.  On success,
\mcode{commitEffect} appends the effect, records the observed state transition,
updates the residual continuation system, and consumes the reservation at one
linearization point.  A failed external operation records failure without
claiming a successful state transition; any recovery goal is separately
authorized.

The separation makes races explicit.  A policy or relevant state-version
change invalidates outstanding reservations and transition witnesses before a
later effect can commit.  An exact retry returns the retained receipt.  A
conflicting retry under the same identifier is denied.  These rules bind
planner-visible goal evolution to operation-time enforcement rather than to a
one-time plan approval.

\subsection{Delegation, joins, and replanning}

Delegation in this profile changes the executing worker, not the authorization
root.  The delegated goal must be an attenuation of its parent, and the worker
receives only the active goal identifier and the fields required to construct
valid requests.  Worker-provided text is never an authority source.

A same-root join names the complete sibling set, a permitted partial order, and
the residual completion.  The reference checker enumerates
registered interleavings and verifies their prefixes.  A new sibling cannot be
added to an existing witness; the composition is rechecked and a new witness
is issued.  Cross-root joins are rejected in this paper's profile.

Replanning uses \mcode{replaceGoal}.  The replacement is checked against the
parent's current residual language after the already committed prefix, not
against the parent's original start state.  Thus a recovery plan can change
method while preserving the task, but it cannot regain resources, erase an
obligation, or make an already-forbidden state reachable.

Cancellation is a tombstone, not completion evidence.  It removes the
cancelled node's future effects and stales its witness, but cannot by itself
justify closing the root.  Recovery is explicit: \mcode{replaceGoal} may target
the cancelled node, must pass the complete current-version transition check,
and leaves the old node marked replaced.  The replacement then follows the
ordinary effect and closure rules.

\subsection{Receipts and auditability}

All accepted and rejected operations produce canonical, hash-linked records.
The record contains typed inputs, version identifiers, the checker decision,
and a reason code; it excludes private chain-of-thought.  The independent
checker consumes only the root contract, initial state, registered operation
semantics, and receipts.  It reconstructs the graph, trace, versions,
obligations, and decisions without importing the runtime checker.  This design
turns auditability into a second implementation path rather than a copy of the
same decision function.

\subsection{Fail-closed behavior}

Unknown schema fields, missing state observations, unsupported predicates,
unbounded continuation systems, stale versions, and abstraction ambiguity do
not authorize effects.  They yield \Indeterminate{} or \Reauthorize{} with a
typed reason.  A deployment may seek a more precise abstraction or obtain
principal approval; it may not silently treat uncertainty as authority.

\section{Reference Implementation}
\label{sec:implementation}

The artifact implements the finite-domain profile used by the formal model and
evaluation.  It is intentionally provider-independent: agent runtimes connect
through adapters, while the authorization decision and deterministic state
oracles remain external to the model.

\subsection{Typed state and trace representation}

Root contracts and goals use canonical JSON with schema-checked identifiers,
predicates, resources, obligations, versions, and finite-state transitions.
Normalized effects contain an operation name and a closed set of typed
arguments.  Canonical serialization sorts object keys, rejects duplicate or
unknown security-relevant fields, and hashes the resulting byte sequence with
domain separation.  Hashes provide stable identifiers and tamper evidence;
caller authentication remains a deployment responsibility stated separately
from content binding.

The reference domain represents environment state as finite typed variables
and each registered effect as one deterministic state transformer or a typed
rejection.  The mathematical model permits nondeterministic transition
relations, but the executable evidence reported here uses this exact
deterministic finite profile.  Policy outcomes are consequently reproducible
without a language model as the primary judge.

The implementation stores maximal completion traces rather than materializing
every prefix.  If $M_g$ is the stored list, the semantic language in
Section~\ref{sec:formalization} is $K_g=\mathsf{Pref}(M_g)$ and the closing set
$D_g$ contains the maximal traces in $M_g$.  After each effect, every node on
the selected witness path is residualized; an empty stored residual marks
completion of that path.  Table~\ref{tab:formal-implementation-map} records the
remaining representation choices so that an implementation convenience is not
mistaken for a stronger formal claim.

\begin{table*}[t]
\caption{Mapping from the formal model to the finite reference profile.}
\label{tab:formal-implementation-map}
\small
\begin{tabularx}{\textwidth}{L{0.18\textwidth}L{0.34\textwidth}Y}
\toprule
Formal object & Reference representation & Enforced boundary \\
\midrule
$L_C,K_g,D_g$ & finite lists of maximal completion traces; prefix closure is implicit & exact order and multiplicity; an empty residual is required for closure \\
$\mathsf{Post}$ & deterministic typed state transformer per registered effect & preconditions, forbidden predicates, frames, and current policy are rechecked \\
$K_{X,w}$ & one fresh leaf-to-root witness path and all ancestor residuals & a graph revision stales other paths; only explicit join composes siblings \\
$\mathsf{Comp}_{\prec}$ & all effect-level shuffles preserving component order and registered precedence & every enumerated prefix and terminal state is checked \\
$R_g$ & finite resource identifiers with set inclusion & no quantitative concurrent-resource claim \\
$O_g$ & parent obligations minus catalog-validated discharge & a child cannot mint an obligation or erase one by replacement \\
$\nu$ and $q$ & immutable root epoch, policy/state/graph counters, and root lifecycle state & stale, suspended, and closed executions cannot commit effects \\
\bottomrule
\end{tabularx}
\end{table*}

\subsection{Transition checker}

The checker validates structure and versions before exploring the proposed
continuation system.  The implemented exact profile enumerates every
registered maximal continuation, its reachable prefix states, every explicit
join shuffle, and every registered closing condition.  The separate bounded
model retains predecessor links for minimal counterexamples.  The formal
semantics also admits sound over-approximation, but the artifact does not claim
an implemented abstraction-refinement engine or measure its precision.  Its
\Indeterminate{} outcome is exercised for missing trusted observations; it is
never converted to \Allow{}.

Success refinement is checked independently from trace safety.  Closing a
parent requires its success predicate in the common terminal state after
children and residual work, with every inherited obligation discharged and
every active frame predicate preserved along the witnessed prefixes.  This
prevents a child from satisfying a convenient proxy while leaving the parent
task incomplete.

Exact sibling composition has an explicit finite cost.  For incomparable
components with fixed trace lengths $\ell_1,\ldots,\ell_m$, the unconstrained
upper bound on effect-level shuffles is
\begin{equation}
 \frac{(\sum_i \ell_i)!}{\prod_i \ell_i!};
 \label{eq:shuffle-cost}
\end{equation}
registered precedence edges reduce this number.  The reference profile
enumerates the finite set and fails closed when its registered bounds are
exceeded.  Larger deployments require symbolic partial-order reduction or
domain-specific solvers rather than treating this enumerator as a scalable
general planner.

The finite envelope has two distinct layers.  Directly represented inputs are
limited to 64 retained goals, 256 maximal continuations per goal, 64 effects
per continuation, 256 entries in each goal-constraint array, eight join
sources, and 28 supplied precedence-edge entries.  The exact join checker also
guards its internal enumeration at 4,096 topological orders, 4,096 products of
source-trace choices, and 4,096 raw composed sequences.  These three 4,096
limits are internal work guards, not joined-goal output capacities: any
successfully authorized joined goal remains subject to the 256-continuation
representation limit.  A chain of depth 63 fills the 64-goal population cap;
attempting depth 64 requires a sixty-fifth goal rather than crossing a separate
depth parameter.

\subsection{Reference monitor and store}

The monitor exposes the operations in Table~\ref{tab:operations}.  It retains
the root contract, goal records, witnesses, effect reservations, the committed
trace, state snapshots, and hash-linked receipts.  Updates are serialized in
the reference profile.  Root lifecycle is explicit: a relevant version change
suspends the execution, successful current-version reauthorization resumes it,
and validated root closure is terminal.  A deployment with concurrent or
distributed writers must preserve the same linearization and freshness conditions; the present
artifact does not reinterpret file serialization as a consensus result.

Effect reservations bind the expected trace and state versions.  Commit
compares both values, checks the normalized effect byte-for-byte, applies the
registered transition, and writes the receipt with the updated residual state.
The adapter receives only the exact operation released by this step.

\subsection{Independent checker}

The audit checker is separately implemented and imports no policy-decision,
reachability, graph-update, or expected-outcome function from the reference
monitor.  It parses the receipt stream, recomputes canonical digests, rebuilds
the graph and committed trace, re-evaluates registered predicates, and verifies
the correspondence between reservations and committed effects.  It also checks
that every active witness uses the current policy and relevant state versions.

The mutation suite changes semantic fields and then consistently recomputes
outer record hashes.  Consequently, rejection cannot rely only on a broken
checksum.  Registered mutants cover omitted mediation, stale versions, trace
reordering, multiplicity changes, dropped obligations, weakened frame
predicates, broadened resources, replaced success predicates, cross-root
delegation, sibling-composition omission, detached commits, and conflicting
idempotency replays.

\subsection{Runtime adapters}

Two adapters map public tool-execution events into the same protocol.  Each
adapter documents the frozen upstream version, license, source digest,
invocation path, and field mapping.  An adapter is counted as a runtime path
only when the actual upstream component constructs and transmits a tool request
through its native call path.  The evaluated components are AgentDojo's
function runtime and $\tau$-bench's environment/tool dispatch; they are public
evaluation stacks rather than claimed production deployments.  Feeding a
hand-authored trace directly to the reference monitor is reported as
deterministic replay, not runtime integration.

The adapter boundary is narrow.  For each frozen case, the upstream runtime
dispatches a typed pair consisting of a case identifier and the next operation
index through its native tool-call path.  A persistent local broker rejects
out-of-order indexes and maps the pair to exactly one pre-registered protocol
operation, which may be a graph transition, effect reservation, effect commit,
version change, cancellation, or closure.  Witnesses remain in the broker's
reference state; this experiment does not claim that either upstream runtime
natively stores them in a planner context.  Both adapters target the same
local deterministic service, so external providers or live websites cannot
change the policy oracle.  Runtime-specific dispatch and transport code remain
outside the trusted decision function but are included in the executed mapping
evidence.

\subsection{Reproducibility}

The retained artifact contains the registered contracts, workflow source manifest,
reference monitor, checker, baselines, bounded model, mutant registry, runtime
adapters, raw JSON receipts, the secondary scale protocol and records,
generated tables, and verification scripts.  One
verification command reruns the deterministic reference cases and checkers,
validates the frozen upstream-runtime, performance, and scale-sensitivity
receipts, validates fixed denominators, recomputes result digests, and fails if
generated manuscript macros differ from the retained results.  Reacquiring
upstream repositories and retiming either host-local study are explicit,
separate commands; neither a model provider nor credentials are prerequisites
for the reported decisions.

\section{Evaluation Design}
\label{sec:evaluation-design}

The evaluation separates seven evidence classes: (1) a minimal separating
counterexample; (2) bounded exploration of the formal transition system;
(3) conformance, mutation, and independent-checker tests; (4) controlled
matched workflows; (5) structural mappings to frozen public task sources; and
(6) calls through two frozen public tool-execution implementations; and (7) a
secondary finite-envelope sensitivity study.  Evidence is not
pooled across these classes.  In particular, a public-source mapping is not
reported as execution of that benchmark, and bounded exploration is not
reported as a proof of an unbounded theorem.  The scale cells are not counted
as external-source or upstream-runtime evidence.

\subsection{Research questions}

\begin{description}[leftmargin=3.6em,style=nextline]
  \item[RQ1: Separation.] Can a locally permitted effect participate in a
  trajectory whose composition violates the approved root task contract?

  \item[RQ2: Prevention.] Does transition authorization prevent registered
  non-refining graph edits, stale-witness uses, and invalid effect orders before
  a forbidden state is committed?

  \item[RQ3: Benign utility.] Does the mechanism retain the matched valid
  continuation for every registered workflow?

  \item[RQ4: Strong comparison.] Does a distinction remain when a
  history-aware complete-continuation trace check receives the root trace
  policy, effect catalog and transition semantics, forbidden and required
  terminal predicates, current state, complete committed prefix and proposed
  continuation, and proposed goal success predicate?

  \item[RQ5: Portability and cost.] Do two independently developed public
  tool-execution paths preserve the registered decisions, and what local latency and
  storage costs accompany the reference implementation?

  \item[RQ6: Finite-envelope sensitivity.] Across fixed host-local
  one-factor-at-a-time grids, how do transition-decision latency, serialized
  evidence and retained-state size, and runtime memory indicators vary, and do
  the declared direct and internal guards return their registered outcomes?
\end{description}

\subsection{Controlled workflow registry}

The registry contains 48 contextual workflow templates, each with one benign
and one drifted case, for 96 cases in total.  The pair shares its root contract,
initial state, operation vocabulary, and available tool affordances.  It
differs in an authorization-relevant graph transition or effect order.  Eight
classes are represented: create, replace, replan, same-root delegation,
same-root composition, stale witness, invalid effect order, and resource-scope
expansion.  Software, data, administrative, and infrastructure workflows
supply the operational settings; a separate access-grant case exercises
same-root composition.

Twenty-four templates are structurally mapped one-to-one to frozen identifiers
from SWE-bench Verified, AgentDojo, $\tau$-bench, and OSWorld
\cite{jimenez2024swebench,debenedetti2024agentdojo,yao2025taubench,xie2024osworld}.  For each source,
the manifest records its repository commit and tree, license digest, task
identifier, source-file digest, and the exact derivation boundary.  The source
contributes task shape and operation vocabulary only.  Original prompts,
benchmark outcomes, and expected answers are not copied, and the matched goal
transitions are paper-specific controlled fixtures.

Contextual variation is not treated as independent replication.  A canonical
structure analysis removes lexical task, effect, resource, and state names
while retaining transition class, trace shape, dependencies, guards,
resources, obligations, version changes, and composition.  Results therefore
report both contextual templates and distinct structural schemas.  Corpus
claims are finite-set conformance statements; no population prevalence is
inferred from repeated contextual instances.

\subsection{Compared controls}

Every case is replayed under five comparators and the complete mechanism:

\begin{enumerate}
  \item a \emph{memoryless action allowlist}, which checks only whether the
  current normalized effect is listed;
  \item a \emph{static goal-effect set}, which also checks the resources
  assigned to the proposed goal but retains no history or graph composition;
  \item a \emph{per-step declared-purpose check}, which receives the current
  structured purpose but does not validate the composed goal graph;
  \item a \emph{final-state check}, which evaluates the registered policy after
  the attempted workflow and therefore distinguishes detection from prevention;
  \item a \emph{history-aware complete-continuation trace check}, which receives
  the root trace language, effect catalog and transition semantics, forbidden
  predicates, required terminal predicate, complete committed prefix, current
  state, full proposed continuation, and proposed goal success predicate; and
  \item \emph{complete transition authorization}, including typed lineage,
  resources, obligations, version freshness, graph composition, and the
  operation-time effect gate.
\end{enumerate}

The complete-continuation comparator is intentionally strong within its trace
projection.  It checks the proposed traces and their reachable states before
execution, rechecks next-state forbidden predicates during commit, and checks
the required terminal predicate at closure.  It therefore prevents every
registered unsafe trace prefix expressible in its supplied finite model.  Its
documented boundary is that it does not authorize typed resource, obligation,
lineage, version, or explicit join-evidence changes unless they alter that
projection.  This comparison isolates the incremental value of treating graph
mutation itself as an authorization event.

\subsection{Public tool-execution paths}

The portability matrix selects one matched pair from each transition class,
giving 16 cases per runtime.  AgentDojo requests traverse its native
\mcode{FunctionsRuntime.run\_function} path; $\tau$-bench requests traverse
\mcode{Env.step} and \mcode{Tool.invoke}.  A run counts as runtime evidence only
when the frozen upstream component constructs and dispatches the tool request
through that path.  Directly feeding a handwritten trace to the reference
monitor remains deterministic replay and is not counted as integration.

Both adapters use the same local deterministic effect service and require no
model provider, network service, credential, or live website.  Each retained
run binds the upstream commit and tree, license and source digests, adapter
digest, case identifier, reset status, complete authorization receipts, final
state, and independent-validator outcome.  A benign and drifted case are reset
independently; a crash, malformed request, version mismatch, or validator
disagreement remains a failed run rather than receiving an inferred verdict.

\subsection{Secondary finite-envelope study}

The secondary study uses deterministic synthetic contracts and contains
\GoalTransitionScaleCellCount{} fixed cells: 100 non-join cells over goal
population (G), chain depth (D), trace length (L), continuation count (T),
resources (R), obligations (O), active-frame predicates (F), closing
predicates (C), and current-policy predicates (P); 23 join cells over source
count (J), precedence density (Q), and trace-choice/output amplification (A);
and six cells at or around internal enumeration guards.  This is a
finite-profile characterization, not a new task benchmark or external
application layer.  G uses depth-one siblings whereas D uses one chain, so
their measurements are not treated as the same graph structure.  P=256 is the
predeclared observation ceiling, not an implementation limit.  The A family
co-varies source-trace choice product and normalized output count and therefore
does not isolate either factor causally.

Before the final run, the protocol, cell registry, expected decision and
reason, order seed, implementation hashes, analysis rules, and host preflight
were frozen.  Each cell runs in five process blocks, with 10 warm-up and 40
retained measurements per block, for 200 latency samples per cell.  Every
sample creates a fresh runtime; fixture construction and setup are outside the
timed region.  One untimed correctness record and one isolated memory
subprocess are retained per cell.  The fixed randomized block order uses a
monotonic clock, nearest-rank median/p95/p99, block-median absolute deviation,
and no outlier deletion.  Decision classes are reported separately rather than
combined into one latency curve.  The study neither fits an asymptotic model
nor performs population significance testing.

\subsection{Outcomes and analysis}

Primary safety outcomes are a forbidden state committed, a drifted transition
accepted, stale-witness use, and protected effects committed before
intervention.  Primary utility outcomes are benign completion, benign
rejection, and \Indeterminate{} decisions.  Runtime and systems outcomes are
transition-decision and effect-decision latency, receipt bytes, and retained
state bytes.  RQ6 additionally retains complete-fixture, target-operation,
receipt-chain, and retained-state canonical JSON bytes; join-enumeration
diagnostics; and process memory snapshots before, immediately after, and after
garbage collection.

Ground truth is computed from the registered finite transition semantics.  A
separately implemented checker reconstructs graph, version, trace, reservation,
and terminal-state properties from receipts without importing the runtime's
policy-decision or reachability functions.  Directed mutations recompute outer
hashes after changing semantic fields, so rejection cannot rely on detecting a
broken checksum alone.

Results are exact counts over the frozen finite registries.  The primary
operation-cost study and secondary finite-envelope study use separate
denominators.  Performance uses a
monotonic clock after correctness gates pass, with fixed warm-up and measured
iteration counts on one recorded host.  Median, p95, and p99 characterize this
reference implementation only.  They are not used to rescue a failed safety
gate and are not extrapolated to upstream production deployments.

\subsection{Evidence gates}

The deterministic evidence freeze passes only if all registered drifted cases
avoid forbidden commits, all benign cases complete, the producer and
independent checker agree on every retained record, all semantic mutants are
rejected, every public-source mapping resolves to exactly one frozen anchor,
and every fixed denominator and generated manuscript value reproduces.  The
runtime layer additionally requires all 32 native-path runs to complete with
the expected authorization outcome and valid receipts.  The secondary study
passes only if all 14 input-contract preflight cases and all
\GoalTransitionScaleCellCount{} cell-level decision/reason expectations pass,
all retained records and raw sample denominators verify independently, and the
result binds the frozen execution inputs and host preflight.

The local protocol was frozen before the final evidence run; the dated protocol
and frozen component digests are retained with the artifact.  The manuscript
accordingly uses the precise designation \emph{registered study protocol}; the
complete protocol, retry rules, missing-run treatment, and evidence boundaries
are included with the artifact.

\section{Results}
\label{sec:results}

All reported results passed the fixed-denominator verifier before manuscript
generation.  Table~\ref{tab:evidence-inventory} gives the primary denominators.
The workflow registry is a deliberately constructed finite conformance corpus,
not a probability sample; results are therefore reported as exact counts over
the registry rather than as population estimates.

\begin{table}[t]
  \centering
  \caption{Verified evidence inventory.  Public-source seeds contribute task
  identifiers, task shape, and operation vocabulary, not source outcomes.}
  \label{tab:evidence-inventory}
  \small
\begin{tabular}{lr}
\toprule
Evidence item & Count \\
\midrule
Canonical structural schemas & 25 \\
Contextual workflow instances & 48 \\
Matched benign/drifted cases & 96 \\
Public source-derived scenario seeds & 24 \\
Mechanism--case decisions & 576 \\
Protocol conformance checks & 65 \\
Public upstream runtime cases & 32 \\
Native upstream operation dispatches & 258 \\
Forbidden commits, complete mechanism & 0 \\
Benign completions, complete mechanism & 48 \\
Forbidden commits, memoryless comparator & 33/48 \\
\bottomrule
\end{tabular}

\end{table}

\subsection{RQ1: separation and formal checks}

The matched counterexample in Section~\ref{sec:counterexample} separates a
memoryless action decision from a history-dependent task constraint.  Both
effects are individually listed, but the pair reaches the forbidden state
unless the intervening review obligation is discharged.  The executable
finite model corroborates this separation: the complete profile reaches
\GoalTransitionBoundedStateCount{} states through
\GoalTransitionBoundedTransitionCount{} transitions without an invariant
violation.  A separately coded static oracle agrees with all 13 initial
proposal decisions and their composed trace sets, including all four
admissible proposals.  Each of 13 single-defect formal variants has a retained
counterexample.

The controlled corpus shows the same distinction at a larger finite scale.
The memoryless allowlist enters a registered forbidden state in
\GoalTransitionMemorylessUnsafeCommitCount{} of
\GoalTransitionMemorylessDriftedCaseCount{} drifted cases.  This is a
constructive result for the stated memoryless decision class and the registered
sequence constraints; it is not used to characterize every access-control
system.

\subsection{RQ2 and RQ3: prevention and benign completion}

The complete mechanism finishes all
\GoalTransitionCompleteBenignCompletionCount{} benign cases and commits no
forbidden state.  Among the \GoalTransitionDriftedCaseCount{} matched drifted
cases, \GoalTransitionCompleteDeniedCount{} return \Deny{} and
\GoalTransitionCompleteReauthorizationCount{} return \Reauthorize{}; none
completes.  The decisions cover eight transition classes: nine create, five
delegate, nine replace, and nine replan cases are denied for trace expansion;
five resource-scope cases are denied for scope expansion; one explicit join is
denied for incomplete composition; five invalid-order cases are denied at the
effect gate; and five stale-witness cases require reauthorization.

Intervention timing distinguishes prevention from post-commit detection.
\GoalTransitionBlockedBeforeEffectCount{} drifted graph edits are denied before
any protected effect in their workflow.  The five stale-version cases suspend
the root and require reauthorization before a later protected effect.  The
\GoalTransitionInvalidOrderAfterSafePrefixCount{} invalid-order cases execute
only their permitted prefix before the next effect is denied.  Thus every
drifted case is stopped before its registered forbidden completion, while the
matched benign member remains executable.

\subsection{RQ4: comparison with trace and local controls}

Table~\ref{tab:control-results} reports all
\GoalTransitionDecisionCount{} mechanism--case decisions.  ``Benign'' is the
number of 48 benign cases completed.  The remaining columns concern the 48
drifted cases: committed, denied, referred for reauthorization, detected only
after commit, entering a forbidden state, and preventively blocked.  The
comparators are logical control classes over the projections specified in
Section~\ref{sec:evaluation-design}; they are not reimplementations of the
cited learned monitors or planning systems.

\begin{table*}[t]
  \centering
  \caption{Outcomes by control.  ``Commit'' does not necessarily mean a
  forbidden state: a drift can violate typed lineage, resource, or freshness
  requirements without changing the trace projection.}
  \label{tab:control-results}
  \scriptsize
  \setlength{\tabcolsep}{3.6pt}
\begin{tabular}{lrrrrrrr}
\toprule
Mechanism & Benign & Commit & Deny & Reauth. & Post-detect & Forbidden & Prevent \\
\midrule
Memoryless action allowlist & 48 & 44 & 4 & 0 & 0 & 33 & 4 \\
Static goal effect set & 48 & 39 & 9 & 0 & 0 & 33 & 9 \\
Per-step declared purpose & 48 & 39 & 9 & 0 & 0 & 33 & 9 \\
Final-state check & 48 & 11 & 4 & 0 & 33 & 33 & 4 \\
History-aware complete-continuation check & 48 & 11 & 37 & 0 & 0 & 0 & 37 \\
Complete transition authorization & 48 & 0 & 43 & 5 & 0 & 0 & 48 \\
\bottomrule
\end{tabular}

\end{table*}

The history-aware complete-continuation trace check prevents
\GoalTransitionHistoryPreventiveBlockCount{} drifted cases and commits
\GoalTransitionHistoryDriftedCommittedCount{}.  Those 11 comprise five
resource expansions, five stale-witness uses, and one join whose missing typed
composition evidence does not alter the comparator's trace projection.  It
enters no registered forbidden state.  The complete mechanism blocks those 11
through resource, freshness, and explicit join-evidence checks, respectively;
its lineage and obligation checks are exercised elsewhere in the conformance
and mutation suites.  This comparison isolates coverage of the stated fields;
it does not establish empirical superiority over a different runtime monitor,
planner, or learned classifier.

\subsection{Implementation conformance and independent checking}

All \GoalTransitionProtocolCheckCount{} protocol checks pass: 
\GoalTransitionProtocolPositiveCheckCount{} positive and
\GoalTransitionProtocolNegativeCheckCount{} negative cases exercise
\GoalTransitionProtocolGateCount{} distinct gates.  The gates include root
lifecycle, current closing witnesses, bottom-up descendant resolution,
cancellation followed by explicit recovery, obligation non-expansion,
all-interleaving join checks, atomic commit, idempotency, and finite-profile
limits.  Oversized inputs return \Indeterminate{} with a profile-limit reason
rather than being partially enumerated.

The independently implemented checker reconstructs all
\GoalTransitionMatchedCaseCount{} records and semantically replays
\GoalTransitionSemanticallyReplayedReceiptCount{} receipts without importing
the reference runtime's decision code.  It rejects all
\GoalTransitionCheckerMutantCount{} mutants from
\GoalTransitionCheckerMutationFamilyCount{} families, including
\GoalTransitionResealedReceiptMutantCount{} mutations whose outer hashes were
consistently recomputed and \GoalTransitionIndependentProfileLimitMutationCount{}
independently oversized inputs.  These results exercise agreement and
tamper-detection paths for the registered schema; the mathematical safety claim
continues to rest on the assumptions in Section~\ref{sec:formalization}.

\subsection{RQ5: upstream paths and local cost}

The two frozen upstream paths execute \GoalTransitionRuntimeCaseCount{} cases
through \GoalTransitionRuntimeDispatchCount{} native per-operation dispatches.
AgentDojo uses \mcode{FunctionsRuntime.run\_function}; $\tau$-bench uses
\mcode{Env.step} followed by \mcode{Tool.invoke}.  All
\GoalTransitionRuntimeDecisionMatchCount{} case decisions and
\GoalTransitionRuntimeReceiptChainCount{} receipt chains match the reference
records.  The \GoalTransitionCrossRuntimeRecordCount{} paired case records and
\GoalTransitionCrossRuntimeStateCount{} retained states also match across the
two paths.  This establishes the specified provider-free dispatch integration;
it does not treat the structural task mappings as benchmark execution.

Performance is measured locally over
\GoalTransitionPerformanceSchemaCount{} structural representatives after
warm-up, yielding \GoalTransitionPerformanceSampleCount{} timed operations on
an Apple M4 Max.  Table~\ref{tab:local-cost} reports descriptive latency and
median receipt size for the three protocol stages.  Initialization, upstream
model latency, and network latency are outside the timed region.

\begin{table}[t]
  \centering
  \caption{Host-local cost of the deterministic finite profile.  Latencies are
  microseconds; receipt size is canonical UTF-8 JSON bytes.}
  \label{tab:local-cost}
  \small
  \begin{tabular}{lrrrr}
    \toprule
    Stage & Median & p95 & p99 & Receipt B \\
    \midrule
    Transition authorization
      & \GoalTransitionTransitionMedianMicroseconds
      & \GoalTransitionTransitionPercentileNinetyFiveMicroseconds
      & \GoalTransitionTransitionPercentileNinetyNineMicroseconds
      & \GoalTransitionTransitionReceiptBytes \\
    Effect authorization
      & \GoalTransitionEffectAuthorizationMedianMicroseconds
      & \GoalTransitionEffectAuthorizationPercentileNinetyFiveMicroseconds
      & \GoalTransitionEffectAuthorizationPercentileNinetyNineMicroseconds
      & \GoalTransitionEffectAuthorizationReceiptBytes \\
    Effect commit
      & \GoalTransitionEffectCommitMedianMicroseconds
      & \GoalTransitionEffectCommitPercentileNinetyFiveMicroseconds
      & \GoalTransitionEffectCommitPercentileNinetyNineMicroseconds
      & \GoalTransitionEffectCommitReceiptBytes \\
    \bottomrule
  \end{tabular}
\end{table}

Median retained authorization state is
\GoalTransitionRetainedStateAfterTransitionBytes{} bytes after a transition
decision and \GoalTransitionRetainedStateAfterEffectBytes{} bytes after an
effect commit.  These measurements characterize the retained serialized
reference profile; they are not throughput or production-capacity estimates.

\subsection{RQ6: host-local finite-envelope sensitivity}

All \GoalTransitionScaleCellCount{} frozen cells match their fixed decision and
reason, and the independent verifier accepts all
\GoalTransitionScaleMeasuredExecutionCount{} raw latency samples,
\GoalTransitionScaleMemoryExecutionCount{} isolated memory records, and 14
input-contract preflight cases.  No blocked cell changes retained authorization
state.  These outcomes are separate from the controlled-corpus and upstream
runtime results.  Table~\ref{tab:scale-boundaries} reports the direct and
internal boundary probes, and Figure~\ref{fig:scale-latency-ratios} reports
within-family \Allow{} latency ratios.

\begin{table*}[t]
  \centering
  \caption{Observed finite-profile boundaries and internal enumeration-guard probes. Values in parentheses are the realized governed dimension, except internal probes, which report the theoretical enumeration count checked by that cell. Every observed decision and reason matched its registered expectation; blocked cells left retained state unchanged.}
  \label{tab:scale-boundaries}
  \scriptsize
  \setlength{\tabcolsep}{4pt}
  \begin{tabularx}{\textwidth}{@{}L{1.28in}YY@{}}
    \toprule
    Dimension & At/below or ceiling probe; observed result & Contrast probe; observed result \\
    \midrule
Goal population & \mcode{G-064} (64); \Allow & \mcode{G-065} (65); \Indeterminate; \mcode{E\_PROFILE\_LIMIT} \\
Depth below root & \mcode{D-063} (63); \Allow & \mcode{D-064} (64); \Indeterminate; \mcode{E\_PROFILE\_LIMIT} \\
Maximum trace length & \mcode{L-064} (64); \Allow & \mcode{L-065} (65); \Indeterminate; \mcode{E\_PROFILE\_LIMIT} \\
Continuation count & \mcode{T-256} (256); \Allow & \mcode{T-257} (257); \Indeterminate; \mcode{E\_PROFILE\_LIMIT} \\
Resource count & \mcode{R-256} (256); \Allow & \mcode{R-257} (257); \Indeterminate; \mcode{E\_PROFILE\_LIMIT} \\
Obligation count & \mcode{O-256} (256); \Allow & \mcode{O-257} (257); \Indeterminate; \mcode{E\_PROFILE\_LIMIT} \\
Frame-predicate count & \mcode{F-256} (256); \Allow & \mcode{F-257} (257); \Indeterminate; \mcode{E\_PROFILE\_LIMIT} \\
Closure-predicate count & \mcode{C-256} (256); \Allow & \mcode{C-257} (257); \Indeterminate; \mcode{E\_PROFILE\_LIMIT} \\
Policy-predicate study range & \mcode{P-256} (256); \Allow; study ceiling & ---; no profile-cap inference \\
Join-source count & \mcode{J-008} (8); \Allow & \mcode{J-009} (9); \Indeterminate; \mcode{E\_PROFILE\_LIMIT} \\
Normalized join output & \mcode{A-128} (256); \Allow & \mcode{A-129} (258); \Indeterminate; \mcode{E\_PROFILE\_LIMIT} \\
Raw precedence edges & \mcode{IPE-EXACT} (28); \Allow & \mcode{IPE-OVER} (29); \Indeterminate; \mcode{E\_PROFILE\_LIMIT} \\
Trace-choice enumeration & \mcode{ITC-EXACT} (4096); \Deny; \mcode{E\_JOIN\_COMPOSITION} & \mcode{ITC-OVER} (5120); \Indeterminate; \mcode{E\_PROFILE\_LIMIT} \\
Topological-order enumeration & \mcode{ITO-BELOW} (720); \Allow & \mcode{ITO-OVER} (5040); \Indeterminate; \mcode{E\_PROFILE\_LIMIT} \\
    \bottomrule
  \end{tabularx}
  \parbox{\textwidth}{\vspace{2pt}\footnotesize
  \textit{Interpretation.} \mcode{P-256} is only the registered study ceiling, not a policy-predicate cap. \mcode{ITC-EXACT} enumerates 4,096 trace choices/raw compositions and then returns \Deny/\mcode{E\_JOIN\_COMPOSITION} because its supplied one-continuation proposal is not the complete 4,096-continuation composition; it therefore does not establish 4,096 as accepted join-output capacity. The topological-order probes observe 720 below the guard and 5,040 over it; there is no exact 4,096-order fixture. \mcode{IPE-OVER} supplies 29 raw edges, so the raw-length guard fires before validation of the fixture's latent cycle.}
\end{table*}

\begin{figure*}[t]
  \centering
  \begin{minipage}[t]{0.49\textwidth}
    \centering
\begin{tikzpicture}[x=1cm,y=0.57cm]
\node[anchor=south,font=\small\bfseries] at (4.275,6.45) {Non-join dimensions};
\node[anchor=east,font=\scriptsize] at (2.63,6) {Goals};
\fill[gtBlue!72] (2.75,5.885) rectangle (2.7559,6.115);
\fill[gtBlue] (2.7559,6) circle (1.15pt);
\node[anchor=west,font=\scriptsize] at (5.92,6) {1.01$\times$};
\node[anchor=west,font=\tiny,text=black!62] at (2.75,5.75) {\mcode{G-064}/\mcode{G-004}};
\node[anchor=east,font=\scriptsize] at (2.63,5) {Depth};
\fill[gtBlue!72] (2.75,4.885) rectangle (2.7782,5.115);
\fill[gtBlue] (2.7782,5) circle (1.15pt);
\node[anchor=west,font=\scriptsize] at (5.92,5) {1.05$\times$};
\node[anchor=west,font=\tiny,text=black!62] at (2.75,4.75) {\mcode{D-063}/\mcode{D-004}};
\node[anchor=east,font=\scriptsize] at (2.63,4) {Trace length};
\fill[gtBlue!72] (2.75,3.885) rectangle (3.9057,4.115);
\fill[gtBlue] (3.9057,4) circle (1.15pt);
\node[anchor=west,font=\scriptsize] at (5.92,4) {7.45$\times$};
\node[anchor=west,font=\tiny,text=black!62] at (2.75,3.75) {\mcode{L-064}/\mcode{L-004}};
\node[anchor=east,font=\scriptsize] at (2.63,3) {Continuations};
\fill[gtBlue!72] (2.75,2.885) rectangle (5.5938,3.115);
\fill[gtBlue] (5.5938,3) circle (1.15pt);
\node[anchor=west,font=\scriptsize] at (5.92,3) {139.80$\times$};
\node[anchor=west,font=\tiny,text=black!62] at (2.75,2.75) {\mcode{T-256}/\mcode{T-004}};
\node[anchor=east,font=\scriptsize] at (2.63,2) {Resources};
\fill[gtBlue!72] (2.75,1.885) rectangle (2.9947,2.115);
\fill[gtBlue] (2.9947,2) circle (1.15pt);
\node[anchor=west,font=\scriptsize] at (5.92,2) {1.53$\times$};
\node[anchor=west,font=\tiny,text=black!62] at (2.75,1.75) {\mcode{R-256}/\mcode{R-004}};
\node[anchor=east,font=\scriptsize] at (2.63,1) {Obligations};
\fill[gtBlue!72] (2.75,0.885) rectangle (3.6955,1.115);
\fill[gtBlue] (3.6955,1) circle (1.15pt);
\node[anchor=west,font=\scriptsize] at (5.92,1) {5.17$\times$};
\node[anchor=west,font=\tiny,text=black!62] at (2.75,0.75) {\mcode{O-256}/\mcode{O-004}};
\draw[->,black!55] (2.75,0.48) -- (5.88,0.48);
\draw[black!55] (2.7500,0.48) -- (2.7500,0.38) node[below,font=\tiny] {1$\times$};
\draw[black!55] (4.0755,0.48) -- (4.0755,0.38) node[below,font=\tiny] {10$\times$};
\draw[black!55] (5.4010,0.48) -- (5.4010,0.38) node[below,font=\tiny] {100$\times$};
\end{tikzpicture}
  \end{minipage}\hfill
  \begin{minipage}[t]{0.49\textwidth}
    \centering
\begin{tikzpicture}[x=1cm,y=0.57cm]
\node[anchor=south,font=\small\bfseries] at (4.275,6.45) {Remaining and join dimensions};
\node[anchor=east,font=\scriptsize] at (2.63,6) {Frames};
\fill[gtPurple!72] (2.75,5.885) rectangle (5.6209,6.115);
\fill[gtPurple] (5.6209,6) circle (1.15pt);
\node[anchor=west,font=\scriptsize] at (5.92,6) {146.51$\times$};
\node[anchor=west,font=\tiny,text=black!62] at (2.75,5.75) {\mcode{F-256}/\mcode{F-004}};
\node[anchor=east,font=\scriptsize] at (2.63,5) {Closure predicates};
\fill[gtPurple!72] (2.75,4.885) rectangle (3.3809,5.115);
\fill[gtPurple] (3.3809,5) circle (1.15pt);
\node[anchor=west,font=\scriptsize] at (5.92,5) {2.99$\times$};
\node[anchor=west,font=\tiny,text=black!62] at (2.75,4.75) {\mcode{C-256}/\mcode{C-004}};
\node[anchor=east,font=\scriptsize] at (2.63,4) {Policy predicates*};
\fill[gtPurple!72] (2.75,3.885) rectangle (3.2643,4.115);
\fill[gtPurple] (3.2643,4) circle (1.15pt);
\node[anchor=west,font=\scriptsize] at (5.92,4) {2.44$\times$};
\node[anchor=west,font=\tiny,text=black!62] at (2.75,3.75) {\mcode{P-256}/\mcode{P-004}};
\node[anchor=east,font=\scriptsize] at (2.63,3) {Join sources};
\fill[gtPurple!72] (2.75,2.885) rectangle (2.9313,3.115);
\fill[gtPurple] (2.9313,3) circle (1.15pt);
\node[anchor=west,font=\scriptsize] at (5.92,3) {1.37$\times$};
\node[anchor=west,font=\tiny,text=black!62] at (2.75,2.75) {\mcode{J-008}/\mcode{J-004}};
\node[anchor=east,font=\scriptsize] at (2.63,2) {Join orderings};
\fill[gtPurple!72] (2.75,1.885) rectangle (4.0329,2.115);
\fill[gtPurple] (4.0329,2) circle (1.15pt);
\node[anchor=west,font=\scriptsize] at (5.92,2) {9.29$\times$};
\node[anchor=west,font=\tiny,text=black!62] at (2.75,1.75) {\mcode{Q-Q0}/\mcode{Q-Q5}};
\node[anchor=east,font=\scriptsize] at (2.63,1) {Join output};
\fill[gtPurple!72] (2.75,0.885) rectangle (5.1790,1.115);
\fill[gtPurple] (5.1790,1) circle (1.15pt);
\node[anchor=west,font=\scriptsize] at (5.92,1) {68.01$\times$};
\node[anchor=west,font=\tiny,text=black!62] at (2.75,0.75) {\mcode{A-128}/\mcode{A-004}};
\draw[->,black!55] (2.75,0.48) -- (5.88,0.48);
\draw[black!55] (2.7500,0.48) -- (2.7500,0.38) node[below,font=\tiny] {1$\times$};
\draw[black!55] (4.0755,0.48) -- (4.0755,0.38) node[below,font=\tiny] {10$\times$};
\draw[black!55] (5.4010,0.48) -- (5.4010,0.38) node[below,font=\tiny] {100$\times$};
\end{tikzpicture}
  \end{minipage}
  \caption{Within-family median latency ratios for selected \Allow{} cells, on a logarithmic horizontal scale. Each bar compares the displayed endpoint cell with its registered same-outcome reference cell (pair printed beneath the bar); fixture construction and setup are outside the timed authorization call. Ratios are local sensitivity measurements, not asymptotic-complexity estimates or cross-family comparisons. The policy-predicate endpoint marked with an asterisk is a study ceiling rather than a profile boundary.}
  \Description{Two horizontal logarithmic bar charts report within-family median authorization-latency ratios. The largest ratios are 146.51 times for frame predicates, 139.80 times for continuation count, and 68.01 times for join output; goal population, depth, resources, and join-source count remain near their registered references.}
  \label{fig:scale-latency-ratios}
\end{figure*}
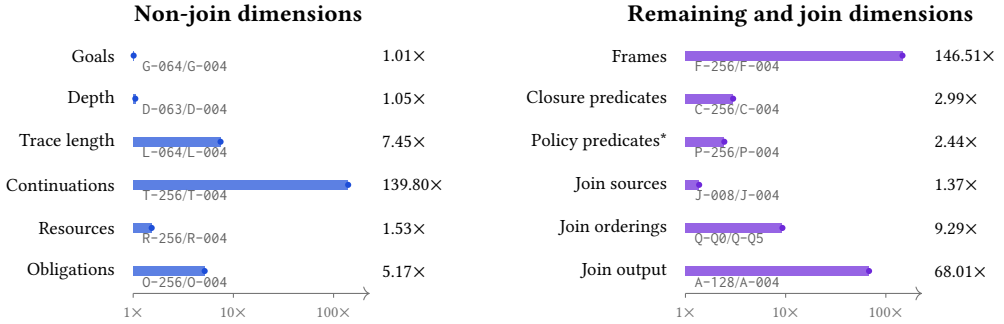

Every declared direct representation boundary in the fixed registry is
accepted, while each corresponding one-over-limit cell returns
\Indeterminate{}/\mcode{E\_PROFILE\_LIMIT}.  This includes 64 retained goals,
a depth-63 chain that fills the 64-goal population, length-64 traces, 256
continuations and entries in each governed goal-constraint array, eight join
sources, 28 supplied precedence-edge entries, and 256 normalized joined
continuations.  The P family ends at its predeclared observation ceiling of 256
current-policy predicates and therefore establishes no separate policy-list
limit.

The internal probes distinguish enumeration work from successful output.
\mcode{ITC-EXACT} reaches 4,096 source-trace choices and 4,096 raw compositions,
then returns \Deny{}/\mcode{E\_JOIN\_COMPOSITION}: the supplied one-continuation
proposal does not equal the complete composition.  It does not authorize a
4,096-continuation joined goal.  The largest successful A cell contains 256
normalized continuations; the next fixed cell supplies 258 and returns
\Indeterminate{} without truncation.  The topological-order construction
observes 720 orders below the guard and fails closed for a fixed 5,040-order
input; no exact-4,096 topological-order poset is claimed.  The 29-edge cell is
rejected by the direct raw-length guard before its latent cycle is evaluated.

Within the \Allow{} class, endpoint-to-reference median ratios are 1.01 for
goal population and 1.05 for chain depth, compared with 7.45 for trace length,
139.80 for continuation count, and 146.51 for active-frame predicates.  The
corresponding ratios are 1.53 for resources, 5.17 for obligations, 2.99 for
closing predicates, and 2.44 at the policy-predicate study ceiling.  For joins,
eight chain-ordered sources are 1.37 times their four-source reference, the
256-output A cell is 68.01 times its reference, and the six-source unordered
Q cell with 720 topological orders is 9.29 times the fully ordered one-order Q
cell.  These are descriptive ratios within each fixed family; adjacent points
need not be monotone, and the design does not support cross-family causal or
asymptotic conclusions.  The median relative block-median absolute deviation
over all cells is \GoalTransitionScaleRelativeBlockMedianMadMedianPercent\%, and
\GoalTransitionScaleRelativeBlockMedianMadAtMostFivePercentCellCount{} of 129
cells are at or below 5\%.  Among \Allow{} cells, the largest p99 is
\GoalTransitionScaleAllowMaximumPercentileNinetyNineMicroseconds{} microseconds
at \GoalTransitionScaleAllowMaximumPercentileNinetyNineCell{} on this host.

Serialized evidence reflects a different axis from decision latency.  At 64
goals, retained state is \GoalTransitionScaleGoalPopulationCapRetainedStateBytes{}
canonical JSON bytes; at 256 joined continuations it is
\GoalTransitionScaleJoinOutputCapRetainedStateBytes{} bytes.  The transition
receipt binds a goal digest instead of embedding the complete goal, whereas
the complete receipt chain grows with graph history: it is
\GoalTransitionScaleGoalPopulationCapCompleteReceiptChainBytes{} bytes at 64
goals and \GoalTransitionScaleMaximumDepthCapCompleteReceiptChainBytes{} bytes
at chain depth 63.  The 28-edge join's target receipt is
\GoalTransitionScalePrecedenceEdgeCapTargetReceiptBytes{} bytes.  Memory records are
diagnostic snapshots: the median immediate heap difference is
\GoalTransitionScaleHeapImmediateMedianBytes{} bytes, the median post-GC heap
difference is \GoalTransitionScaleHeapAfterGarbageCollectionMedianBytes{}
bytes, and the median process-lifetime \mcode{maxRSS} increase is
\GoalTransitionScaleMaximumResidentSetDeltaMedianKilobytes{} KiB.  These values
are neither peak allocation nor portable per-call memory cost.

\section{Discussion}
\label{sec:discussion}

\subsection{Authorization follows planning state}

The central systems implication is that goal-graph mutation is security
relevant when it changes the effects that may become reachable.  A planner may
freely explore internal alternatives, but activating a child, replacement,
delegation, or join is comparable to changing an authorization-relevant
configuration.  Exposing that boundary permits conventional security
techniques---reference monitoring, version checks, append-only receipts, and
fail-closed decisions---to govern long-horizon behavior without attempting to
police private reasoning text.

This separation also clarifies the role of model alignment.  An aligned planner
can reduce the frequency of invalid proposals and improve recovery from a
denial.  The external mechanism supplies a different property: for the
registered effect and state domain, protected operations cannot rely solely on
the planner's own assertion that a new subgoal is faithful.

\subsection{Relationship to plan verification}

Whole-plan verification and transition authorization are complementary.  A
complete finite plan can be checked before execution.  Long-horizon agents,
however, often change the plan after observations, failure, or delegation.
Rechecking the complete model after every change is a strong baseline and may
be the right implementation for small graphs.  Incremental transition checks
become useful when the retained witness identifies which residual language,
invariants, and obligations are affected.  The security property is the same:
no graph update or effect commits without a current proof.  The paper's
contribution is the authorization protocol and invariant, not a claim that
incrementality is universally faster.

\subsection{Deployment patterns}

Coding agents provide one direct application.  A debugging task may authorize
repository reads, a test edit, and execution in a sandbox, while deployment,
credential changes, and production traffic remain outside the root contract.
A generated remediation subgoal can change method without acquiring those
effects.  Administrative agents can similarly rebook or refund within one
customer, amount, and time window while preventing a recovery branch from
widening the account scope.  Data-analysis agents can revise statistical
methods while retaining dataset, purpose, retention, and disclosure
constraints.  Manager--worker systems can delegate within one root task without
allowing a worker's summary to become a new source of authority.

In each setting, the practical requirement is a typed operation and state
boundary.  Existing API gateways, transaction services, CI/CD controllers, or
tool brokers can host that boundary.  The agent runtime needs only to expose
goal mutations and carry witness identifiers with protected calls.

\subsection{Abstraction and operator choices}

Exact finite reachability is not always practical.  A deployment can use
sound abstraction, bounded resource domains, conservative effect summaries,
or domain-specific solvers.  Coarser abstractions tend to increase
\Indeterminate{} decisions rather than unsafe allows, provided the
over-approximation and universal allow rule are preserved.  The retained
counterexample path helps operators refine only the state components that made
a safe proposal ambiguous.

Operators also choose what changes require principal reauthorization.  A
policy version advance, a resource-owner change, or a newly observed state may
invalidate a witness immediately.  Purely observational changes can remain
outside the relevant version set.  That dependency is declared in the root
contract and tested by freshness mutants; it is not inferred from an agent's
description after the fact.

\subsection{Security composition}

Goal-transition authorization does not replace caller authentication,
credential isolation, provenance tracking, input sanitization, rate limits, or
settlement protocols.  It consumes authenticated identities and trusted state
versions, and it assumes the commit adapter is the only path to protected
effects.  In return, it supplies a composition property those controls do not
individually provide: the current goal graph and its still-reachable effects
remain inside one approved structured root task.

\section{Validity Conditions and Scope}
\label{sec:validity}

\paragraph{Construct validity.}
The primary outcomes are defined over registered state predicates, normalized
effects, and trace membership.  They measure authorization within that model,
not every harmful real-world consequence.  Matched benign and drifted cases
hold the prefix and affordances constant to isolate goal evolution, while the
independent checker prevents the runtime implementation from serving as its own
oracle.

\paragraph{Internal validity.}
Expected outcomes, workflow pairs, baseline inputs, mutation operators, and
analysis rules are frozen before the final controlled execution.  Each
run records its initial and final state, complete receipt chain, adapter and
contract digests, and reset result.  Disagreement between validators blocks the
result rather than being resolved by a subjective tie-breaker.  Each comparison
receives the precise input projection stated in
Section~\ref{sec:evaluation-design}.  The strongest trace comparison receives
the root trace policy, effect catalog and transition semantics, forbidden and
required-terminal predicates, current state, complete committed prefix,
complete proposed continuation, and proposed goal success predicate, and it
blocks on unknown outcomes; it does not receive the typed lineage, obligation,
resource, version, or explicit join-evidence checks that define the distinct
question under study.

\paragraph{External validity.}
The study covers four operational domains, one separate access-grant join
fixture, and two public tool-execution paths, with
at least half of the workflow templates rooted in public sources.  These are
deliberately heterogeneous but do not represent every agent framework,
enterprise policy, or physical environment.  Claims are stated for the
registered task families and finite contract profile.  The adapter evidence
tests the upstream dispatch and call paths; it is not evidence of deployment inside the
upstream projects or endorsement by their authors.
The study evaluates enforcement after a typed mutation is presented; proposal
frequency and planning quality are separate outcomes.

\paragraph{Scale-study validity.}
The secondary study is a single-host, one-factor-at-a-time characterization of
synthetic deterministic fixtures.  It does not estimate interactions among
dimensions, production throughput, concurrent or distributed behavior,
end-to-end agent latency, or performance on unseen contracts.  Setup and
fixture construction are excluded from decision latency.  The five block
medians are repeated process blocks on one machine, not independent machines
or population replications, and no confidence interval or significance claim
is made.  A successful boundary cell establishes behavior at that registered
input, not a recommended sustained operating capacity; fail-closed behavior is
established only for the fixed cells.

Each cell uses one isolated memory subprocess.  The reported
\mcode{process.memoryUsage()} values are pre-decision, immediate post-decision,
and post-garbage-collection snapshots; their differences may be negative and
are not peak-allocation estimates.  \mcode{process.resourceUsage().maxRSS} is a
process-lifetime high-water indicator that includes runtime and fixture setup,
not memory attributable solely to the authorization call.  Neither measure is
portable across hosts or runtimes.

\paragraph{Formal validity.}
The safety theorem depends on sound transition semantics, a sound
over-approximation of future traces, fresh relevant versions, complete
mediation, and a faithful commit adapter.  The implementation-to-model table
and targeted mutants exercise these premises separately.  Bounded exploration
supports the finite instances and finds counterexamples to weakened variants;
it does not replace the mathematical proof.  The finite decision property
holds only where reachability is exact and all typed constraints are decidable
in the registered domain.

\paragraph{Structured-policy boundary.}
The mechanism starts from a principal-approved structured root task.  It does
not decide whether arbitrary natural-language goals entail one another and
does not infer permission from free-form agent explanations.  A deployment may
use a model to propose the initial contract, but authority begins only after a
trusted approval step.  This is an interface boundary rather than an empirical
assumption hidden in the evaluation.

\paragraph{Single-root profile.}
Delegation and joins remain under one principal and one authorization root.
Combining independent roots introduces principal conflict, authority algebra,
and cross-root accountability that require a separate model.  The checker
rejects such composition instead of presenting an unproved rule.

\paragraph{Availability and progress.}
Fail-closed behavior permits denial of service through repeated invalid or
ambiguous proposals.  Conditional success preservation establishes what is
true when the root closes; it does not promise that a planner will find a valid
plan or terminate.  The evaluation measures benign completion and recovery so
that this availability cost remains visible.

\paragraph{Unmediated and irreversible effects.}
An effect outside the registered adapter is outside the protected execution.
Likewise, invalidating a witness prevents future commits but cannot reverse an
already completed physical action.  Deployment therefore requires coverage
analysis for effect paths and domain-specific compensation where reversal is
needed.

\section{Conclusion}
\label{sec:conclusion}

Long-horizon tool use makes goal evolution part of the authorization state.
Checking each call against a static permission set cannot enforce policies that
depend on ordered composition, outstanding obligations, explicitly joined
sibling interactions, or changed state.  This paper makes goal creation,
replacement, same-root
delegation, and same-root joining explicit authorization events.  A proposed
transition is activated only when its structured continuation refines the
current root task contract, and every protected effect is mediated against the
resulting residual permitted trace language.

The formal result establishes root-policy safety under explicit, testable
assumptions and preserves required terminal predicates when the authorized
goal graph closes.  The artifact connects these results to bounded exploration,
an independent checker, targeted mutants, six explicitly specified comparison
mechanisms, paired external-source workflows, and two public tool-execution
call paths.  A separately frozen 129-cell sensitivity study additionally
confirms the declared direct and internal finite-profile boundaries and makes
their host-local cost curves explicit without converting them into a production
capacity claim.  Together, these elements provide a planner-independent method for
allowing useful replanning and delegation without letting self-generated
subgoals become a source of new authority within the registered structured
domain and mediation assumptions.

\bibliographystyle{ACM-Reference-Format}
\bibliography{references}

\appendix
\section{Proofs and Executable Formal Audit}
\label{app:formal-proofs}

\subsection{Auxiliary lemmas}

\begin{lemma}[Residual preservation]
\label{lem:residual-preservation}
Let $K$ be prefix closed and suppose
\[
 \mathsf{SafeFuture}_C(s,t,K).
\]
Suppose also
\[
 e^{-1}K\ne\varnothing
 \quad\text{and}\quad
 s'\in\mathsf{Post}(s,e).
\]
Then $te\in L_C$, $s'\notin B_C$, and
$\mathsf{SafeFuture}_C(s',te,e^{-1}K)$ hold.
\end{lemma}

\begin{proof}
Because $e^{-1}K\ne\varnothing$, there is an $eu\in K$.  Prefix closure gives
$e\in K$.  Instantiating Eq.~\eqref{eq:safe-future} with prefix $e$ yields
$te\in L_C$ and
$\mathsf{Post}(s,e)\cap B_C=\varnothing$, hence $s'\notin B_C$.
Now choose arbitrary $u\in e^{-1}K$ and $v\preceq u$.  Since $e u\in K$ and
$e v\preceq e u$, Eq.~\eqref{eq:safe-future} gives
$t e v\in L_C$ and
$\mathsf{Post}(s,e v)\cap B_C=\varnothing$.  By the recursive definition of
$\mathsf{Post}$,
$\mathsf{Post}(s',v)\subseteq\mathsf{Post}(s,e v)$, so its intersection with
$B_C$ is also empty.  These are precisely the conjuncts of
$\mathsf{SafeFuture}_C(s',te,e^{-1}K)$.
\end{proof}

\begin{lemma}[Authorized graph edits preserve the active future invariant]
\label{lem:edit-preservation}
Suppose a runnable configuration receives an edit below an active parent.  If
the checker accepts W1--W7 for $\Delta$ against the current state, committed
prefix, and versions, the installed child--residual composition $K_\Delta$ and
the resulting selected-path intersection satisfy $\mathsf{SafeFuture}_C$,
remain under the same root, and do not expand resources or add or drop
unresolved authorization obligations.
\end{lemma}

\begin{proof}
W1, W2, and W7 establish current identity, same-root lineage, acyclicity, and
version binding.  W3 gives $K_\Delta\subseteq K_p\cap t^{-1}L_C$, binds the
complete child--residual composition, and soundly covers the selected edited
path.  Thus every $u\in K_\Delta$ and prefix $v\preceq u$ satisfies
$tv\in L_C$.  W4 independently quantifies the current-state successors of
every such prefix and supplies the forbidden-state and active-path frame
conjuncts.  The new selected-path language has $K_\Delta$ as one intersection
operand and is therefore a subset of $K_\Delta$; universal safety is preserved
by that restriction even if the edited parent previously belonged to a
different, now-stale selected path.  W5 supplies resource
attenuation and exact inheritance or trusted discharge of parent obligations.
An explicit join includes all registered sibling interleavings in
$K_\Delta$; ordinary stale sibling witnesses authorize no effect.  Together
these facts establish the claimed invariant and side conditions.
\end{proof}

\begin{lemma}[Invalidation excludes stale commits]
\label{lem:invalidation}
Under A1 and A6, after any policy, relevant-state, or goal-graph version
advance, no protected effect can commit using a witness bound to an earlier
version.
\end{lemma}

\begin{proof}
A6 orders the version advance and dependent-witness invalidation before the
next commit.  A policy or relevant-state advance sets the configuration to
suspended.  An accepted graph edit atomically advances $\kappa$, invalidates
the prior witnesses, and installs only the new selected witness while remaining
runnable.  In either case Eq.~\eqref{eq:effect-residual} requires a witness
whose W7 versions equal the trusted current vector, which an old witness does
not satisfy.  A1 excludes a second protected-effect path.
\end{proof}

\subsection{Proof of Theorem~\ref{thm:root-safety}}

\begin{proof}
We induct on the finite transition prefix of an arbitrary execution.

\emph{Base case.}  A7 gives
$\mathsf{SafeFuture}_C(s_0,\epsilon,K_{g_0})$ and therefore
$\epsilon\in L_C$.  The root identity, resource envelope, and obligations are
those of $C$.

\emph{Authorized create, replace, delegation, or same-root composition.}
The operation commits only after W1--W7 hold.  Lemma~\ref{lem:edit-preservation}
establishes the future invariant for the newly selected current-witness path.
The committed trace and external state are unchanged.  The graph-revision
advance makes prior sibling witnesses stale; an explicit join instead checks
all interleavings in Eq.~\eqref{eq:composition}.  W2 excludes a cross-root
edit, while W5 preserves the resource and obligation side conditions.

\emph{Protected-effect commit.}  A1 and A2 ensure that the only executed
protected effect is the normalized $e$ checked by the rule.  W7 and
Lemma~\ref{lem:invalidation} establish freshness.  Because
$e^{-1}K_{X,w}\ne\varnothing$, Lemma~\ref{lem:residual-preservation} gives
$te\in L_C$, excludes every modeled successor in $B_C$, and establishes the
future invariant for $e^{-1}K_{X,w}$.  Every node on the selected path is
residualized at the same commit point.  The commit and residual update are
atomic by A6.  A3 and A4 connect the actual modeled successor and continuation
to the relations quantified by the witness.

\emph{Invalidation.}  Advancing a version commits no effect and does not change
$t$.  The configuration becomes suspended, so the theorem makes no claim that
an exogenously changed world state is safe; it claims that no stale protected
effect follows.  Lemma~\ref{lem:invalidation} supplies that fact.

\emph{Reauthorization.}  A successful reauthorization is the graph-edit case
evaluated for a complete selected path at current versions and returns the
configuration to runnable.  A failed or indeterminate check leaves the
configuration suspended and commits no effect.  When no active non-root goal
exists, the same argument applies to the special current-version graph edit
that installs the first selected path and atomically resumes the root.

\emph{Cancellation or closure.}  Cancellation removes future behavior, creates
no fresh witness, and commits no protected effect.  Goal closure requires a
current witness, a completed registered suffix, and resolved required
descendants.  Either non-root operation advances the graph revision and leaves
no current selected witness; a later protected commit is therefore blocked by
W7 until an active path is reauthorized.  Root closure checks the permitted
completed trace, root predicate, and obligations and makes the configuration
closed.  These operations leave $t\in L_C$; the functional implication of
validated closure is addressed separately in Theorem~\ref{thm:success}.

These cases exhaust the operational rules.  Hence every committed trace stays
in $L_C$, and every protected effect from a runnable configuration is current,
same-root, and root-safe under the modeled transition relation.
\end{proof}

\subsection{Proof of Theorem~\ref{thm:success}}

\begin{proof}
Order closed non-root goal nodes by reverse topological order in the finite
acyclic goal graph.  A leaf can close only when its executed suffix belongs to
its current $D_g$ and its structured predicate $Q_g$ and required discharge
evidence have been validated under a fresh witness.  This is the induction
base.

For the induction step, let parent $p$ have a witnessed refinement whose
required children and residual component have closed.  The induction
hypothesis supplies the validated closing predicates and discharged obligations
for each child.  W6 quantified over the common closing state of every trace in
$D_\Delta$; it therefore entails $Q_p$ in that state.  W4 ensures that the
frames registered for the composition held throughout its prefixes, rather
than only at child-local closing times.  W5 ensures that no unresolved parent
obligation disappeared between levels.  Thus validated child and residual
closure lifts to validated parent closure.

For a branch that closes through descendants, repeated application of this
step yields $Q_{g_0}$ in the common root-closing state.  A cancelled branch
does not enter that induction: it contributes no closure evidence and must be
replaced and completed if the root predicate still depends on its work.  In
all cases the root-close rule directly validates the permitted completed trace,
all root obligations, and $Q_C$; A7 additionally ensures that a validated
$Q_{g_0}$ entails $Q_C$.  Therefore the terminal modeled state satisfies
$Q_C$.  The argument is conditional on closure and does not imply existence of
a closing trace or termination.
\end{proof}

\subsection{Proof of Lemma~\ref{lem:local-insufficiency}}

\begin{proof}
Let the effect alphabet contain $a$ and $b$, let the static allowlist be
$\{a,b\}$, and let the permitted prefix-closed language be
$L=\{\epsilon,a,b\}$.  Both singleton traces are safe and both effects are
individually listed, while $ab\notin L$.

Consider the decision for current effect $b$.  A control whose inputs are only
$b$ and the unchanged flat allowlist receives identical inputs after prefix
$\epsilon$ and after prefix $a$.  If it allows $b$, it preserves the safe
singleton $b$ but also admits forbidden trace $ab$.  If it denies $b$, it
blocks $ab$ but rejects the safe singleton $b$.  Hence it cannot satisfy both
safety and permissiveness on these traces without history or equivalent
composition state.
\end{proof}

\subsection{Proof of Proposition~\ref{thm:relative-completeness}}

\begin{proof}
Every loop in Algorithm~\ref{alg:exact-authorization} ranges over a finite
registered set: graph nodes and topological schedules, finite traces and their
prefixes, finite reachable-state sets, and finite typed constraint domains.
The checker therefore terminates.  Under the proposition's exact-domain
premises, it enumerates every scheduled trace in $K_\Delta$, every trace
prefix, every reachable modeled state, each structural and version equality,
the registered resource aggregates, obligation cases, and every closing state
in $D_\Delta$.

Semantic satisfaction of W1--W7 means that no exact structural test fails, no
concrete unsafe trace or state exists, all success and frame implications and
current-policy constraints hold, and all versions are current.  Such an edit
reaches the final return and yields $\Allow$; no unresolved abstract behavior
exists that could require $\Indeterminate$.  Conversely, reaching the final
return means that every finite check corresponding to W1--W7 succeeded, so the
edit satisfies those obligations.  A stale edit exits with $\Reauthorize$
before enumeration.  Exactness is essential: the statement makes no
completeness claim for a strict over-approximation.
\end{proof}

\subsection{Proof of Proposition~\ref{prop:precision}}

\begin{proof}
Suppose first that the universal $\mathsf{SafeFuture}_C$ check succeeds on
$A_1$.  Every trace and every prefix represented by $A_1$ satisfies the root
trace and forbidden-state predicates.  Since $A_2\subseteq A_1$, the same
universal statements hold on $A_2$, so the subcheck succeeds there as well.
Since $T\subseteq A_i$, either successful abstraction covers every concrete
continuation, proving concrete safety.

If denial is supported by concrete trace $u\in T$, soundness gives
$u\in A_2\subseteq A_1$.  The same retained concrete witness therefore remains
valid under both abstractions.  Finally, if there is no retained concrete
counterexample, removing spurious traces can eliminate the only obstacle to a
universal safety proof.  Set inclusion alone
cannot establish that a remaining abstract trace is concrete; therefore it
cannot create a concrete $\Deny$.  Such a denial requires additional
concretization evidence, a separate information gain rather than abstraction
precision alone.  The conclusion is limited to this universal subcheck; a
separate policy predicate that depends non-monotonically on the representation
of $K$ is not covered.
\end{proof}

\subsection{Bounded executable model}
\label{app:bounded-model}

The dependency-free Node.js model explores the reachable-state fixed point of
a declared finite abstraction containing create, same-root join, delegation,
version advance, invalidation, reauthorization, protected commit, and closure.
The selected complete profile reaches 340 states through 419 transitions, with
maximum shortest-path depth ten and no invariant violation.  An independently
coded static oracle agrees on all 13 initial registered proposal decisions and
their composed trace sets, including allowance of all four semantically
admissible proposals; it does not replay the full lifecycle.  Thirteen
single-defect variants---covering local-only composition,
order, multiplicity, frames, root identity, delegation lineage, resources,
obligations, success refinement, three freshness dimensions, and complete
mediation---are all killed by retained counterexamples.  The abstraction test
checks 256 nested language triples over a four-trace universe.  Prefix closure
is not needed for that set-inclusion regression, while
Proposition~\ref{prop:precision} retains prefix closure as a semantic premise.

These counts are regression assertions for the retained artifact.  They are
not evidence about unbounded programs, unrestricted natural language, or
multi-principal composition.  The audit can be reproduced with
\mcode{node} \path{formal/bounded_model.mjs}; full results and counterexamples are
stored separately so that a summary cannot conceal a failing variant.

\section{Registered Experimental Details}
\label{app:experimental-details}

\subsection{Case schema}

Each workflow record contains: a globally unique case identifier; source
family and source task identifier; license and immutable source reference; root
contract digest; initial typed state; identical prefix; benign and drifted goal
updates; ordered or partially ordered candidate effects; expected transition,
commit, and terminal-state outcomes; drift class; relevant policy/state
versions; and deterministic reset instructions.  The schema rejects unknown
security-relevant fields and requires every expected outcome to be derivable
from the registered state/effect semantics.

\subsection{Baseline interface}

All controls replay the same ordered case definition from a fresh local state,
but each receives only the projection that defines that control.  The
memoryless allowlist receives the normalized effect and root effect catalog;
the static goal-effect control additionally receives the proposed goal's
resources and declared effects; the declared-purpose control also binds the
current effect to that goal.  The final-state control receives the resulting
state and completed trace only after execution.  The history-aware
complete-continuation trace control receives the root trace language, effect
catalog and transition semantics, forbidden and required-terminal predicates,
current typed state, complete committed prefix, proposed completion traces,
and proposed goal success predicate.  The
complete mechanism alone receives typed lineage, obligations, versions,
resources, graph revisions, and transition witnesses.  This design compares
specified control classes rather than claiming identical information or a
reimplementation of each cited system.

\subsection{Run retention}

A run contributes a policy outcome only when the case and runtime versions
match the frozen registry, reset succeeds, all expected records are present,
receipt hashes and denominators verify, and both validators agree.  A crash,
malformed runtime output, version mismatch, reset failure, or validator
disagreement fails the evidence gate and receives no inferred policy outcome;
none occurred in the frozen run.

\subsection{Mutation categories}

The registered mutation set covers the trusted premises represented in the
implementation: effect mediation; root, parent, and worker identity; trace
order and multiplicity; resource scope; success and forbidden predicates;
active frame predicates; inherited obligations; residual completion; policy
and state versions; same-root composition; reservation-to-commit binding; and
idempotency.  Each mutant has a minimal retained counterexample and an expected
rejection reason.  Mutants that merely corrupt an outer hash do not count
toward semantic mutation coverage.

\subsection{Result generation}

All manuscript counts and percentages are produced from verified JSON result
bundles.  The generator checks the registered denominators, emits LaTeX macros
and tables, and records a digest over its inputs.  The manuscript build fails
when a generated file is stale.  Descriptive latency percentiles are derived
from monotonic-clock measurements after correctness gates pass; they are not
used to replace a failed security or utility gate.

\subsection{Secondary scale-study registry and execution}

The baseline non-join fixture has four maximal continuations of length four,
four resources, four obligations, four active-frame predicates, four closing
predicates, and four current-policy predicates.  One factor changes per G, D,
L, T, R, O, F, C, or P family.  Join families separately change source count,
precedence structure, or the combined trace-choice/output construction.  The
registry contains 100 non-join cells, 23 join cells, and six internal-boundary
cells.  G forms depth-one siblings and D forms one chain.  P ends at a study
ceiling of 256.  The Q fixtures use six sources and fixed precedence covers
whose expected raw topological-order counts are 720, 360, 90, 20, 6, and 1.

Every cell is evaluated in five process blocks.  Each block performs 10
warm-ups and retains 40 measurements, yielding
\GoalTransitionScaleWarmupExecutionCount{} warm-ups and
\GoalTransitionScaleMeasuredExecutionCount{} measured calls overall.  A fresh
runtime is used for every measured call; construction and setup are untimed.
One separate correctness execution and one isolated memory subprocess are
retained for each cell.  Cell order is deterministically shuffled from the
fixed seed independently in each block.  Quantiles use nearest rank,
$\max(0,\lceil pn\rceil-1)$, and no sample is removed.

The internal guards remain distinct.  The trace-choice construction reaches
4,096 source-trace choices and 4,096 raw compositions, then returns
\Deny{}/\mcode{E\_JOIN\_COMPOSITION} because the supplied one-continuation
proposal is not the complete composition; it is not a successful
4,096-continuation goal.  Topological-order enumeration is observed at 720 and
at the fixed 5,040-order over-limit construction, not at exactly 4,096.  The
largest successful A-family joined output has 256 normalized continuations;
the next fixed cell supplies 258 and fails closed without truncation.  The
precedence-edge cells accept 28 forward edges and reject a 29-entry input at
the direct length guard.

\subsection{Complete scale-study results}

Table~\ref{tab:scale-all-cells} lists every fixed cell.  Ratios are within-family
median latency ratios to the predeclared reference cell with the same decision
class; a dash denotes an inapplicable comparison.  The table is descriptive
single-host evidence and is not an asymptotic fit.

\begingroup
\scriptsize
\setlength{\tabcolsep}{3.1pt}
\begin{longtable}{@{}lllL{1.42in}rrrrr@{}}
\caption{Complete registered scale-sensitivity matrix (all 129 cells). Latencies are microseconds over 200 retained raw samples per cell. Ratio is the median relative to the registered same-decision family reference; it is omitted when no applicable same-decision reference exists or for over-limit fail-fast cells. State B is canonical retained-state JSON bytes after the measured operation.}\label{tab:scale-all-cells}\\
\toprule
Cell & Factor & Class & Decision/reason & Median & p95 & p99 & Ratio & State B \\
\midrule
\endfirsthead
\caption[]{Complete registered scale-sensitivity matrix (continued).}\\
\toprule
Cell & Factor & Class & Decision/reason & Median & p95 & p99 & Ratio & State B \\
\midrule
\endhead
\midrule
\multicolumn{9}{r}{\footnotesize Continued on next page}\\
\endfoot
\bottomrule
\endlastfoot
\mcode{G-002} & \mcode{2} & interior & \Allow & 168.958 & 208.750 & 255.042 & 0.99 & 6547 \\
\mcode{G-004} & \mcode{4} & interior & \Allow & 170.209 & 212.750 & 264.708 & 1.00 & 15885 \\
\mcode{G-008} & \mcode{8} & interior & \Allow & 168.041 & 220.417 & 273.416 & 0.99 & 34561 \\
\mcode{G-016} & \mcode{16} & interior & \Allow & 172.125 & 209.042 & 267.666 & 1.01 & 71920 \\
\mcode{G-032} & \mcode{32} & interior & \Allow & 168.250 & 204.083 & 321.250 & 0.99 & 146640 \\
\mcode{G-048} & \mcode{48} & interior & \Allow & 170.917 & 208.375 & 482.541 & 1.00 & 221360 \\
\mcode{G-063} & \mcode{63} & interior & \Allow & 170.916 & 208.125 & 245.208 & 1.00 & 291410 \\
\mcode{G-064} & \mcode{64} & at cap & \Allow & 171.958 & 214.917 & 249.416 & 1.01 & 296080 \\
\mcode{G-065} & \mcode{65} & cap+1 & \Indeterminate; \mcode{E\_PROFILE\_LIMIT} & 53.000 & 79.542 & 91.375 & --- & 296095 \\
\addlinespace[2pt]
\mcode{D-001} & \mcode{1} & interior & \Allow & 168.166 & 203.292 & 249.459 & 1.00 & 6547 \\
\mcode{D-002} & \mcode{2} & interior & \Allow & 168.666 & 196.584 & 225.416 & 1.00 & 11266 \\
\mcode{D-004} & \mcode{4} & interior & \Allow & 168.708 & 202.459 & 243.625 & 1.00 & 20704 \\
\mcode{D-008} & \mcode{8} & interior & \Allow & 169.833 & 200.458 & 395.542 & 1.01 & 39580 \\
\mcode{D-016} & \mcode{16} & interior & \Allow & 171.125 & 201.333 & 322.833 & 1.01 & 77340 \\
\mcode{D-032} & \mcode{32} & interior & \Allow & 174.375 & 223.292 & 258.958 & 1.03 & 152860 \\
\mcode{D-048} & \mcode{48} & interior & \Allow & 175.084 & 221.916 & 289.334 & 1.04 & 228380 \\
\mcode{D-062} & \mcode{62} & interior & \Allow & 174.709 & 224.541 & 483.833 & 1.04 & 294460 \\
\mcode{D-063} & \mcode{63} & at cap & \Allow & 177.166 & 228.375 & 263.750 & 1.05 & 299180 \\
\mcode{D-064} & \mcode{64} & cap+1 & \Indeterminate; \mcode{E\_PROFILE\_LIMIT} & 53.917 & 68.916 & 90.917 & --- & 299195 \\
\addlinespace[2pt]
\mcode{L-001} & \mcode{1} & interior & \Allow & 110.958 & 136.209 & 170.084 & 0.65 & 5047 \\
\mcode{L-002} & \mcode{2} & interior & \Allow & 131.583 & 150.667 & 172.500 & 0.77 & 5547 \\
\mcode{L-004} & \mcode{4} & interior & \Allow & 170.916 & 209.917 & 366.666 & 1.00 & 6547 \\
\mcode{L-008} & \mcode{8} & interior & \Allow & 262.625 & 394.500 & 480.625 & 1.54 & 8547 \\
\mcode{L-016} & \mcode{16} & interior & \Allow & 388.292 & 451.542 & 557.292 & 2.27 & 12547 \\
\mcode{L-032} & \mcode{32} & interior & \Allow & 678.542 & 808.584 & 931.625 & 3.97 & 20547 \\
\mcode{L-048} & \mcode{48} & interior & \Allow & 980.875 & 1205.750 & 1313.292 & 5.74 & 28547 \\
\mcode{L-063} & \mcode{63} & interior & \Allow & 1270.625 & 1485.083 & 1579.667 & 7.43 & 36047 \\
\mcode{L-064} & \mcode{64} & at cap & \Allow & 1272.583 & 1478.084 & 1548.167 & 7.45 & 36547 \\
\mcode{L-065} & \mcode{65} & cap+1 & \Indeterminate; \mcode{E\_PROFILE\_LIMIT} & 16.167 & 32.833 & 43.292 & --- & 7895 \\
\addlinespace[2pt]
\mcode{T-001} & \mcode{1} & interior & \Allow & 93.917 & 116.084 & 149.875 & 0.56 & 5017 \\
\mcode{T-002} & \mcode{2} & interior & \Allow & 119.167 & 142.875 & 177.333 & 0.71 & 5527 \\
\mcode{T-004} & \mcode{4} & interior & \Allow & 168.208 & 190.792 & 230.333 & 1.00 & 6547 \\
\mcode{T-008} & \mcode{8} & interior & \Allow & 272.750 & 324.709 & 493.375 & 1.62 & 8587 \\
\mcode{T-016} & \mcode{16} & interior & \Allow & 508.791 & 595.333 & 785.791 & 3.02 & 12667 \\
\mcode{T-032} & \mcode{32} & interior & \Allow & 1077.375 & 1245.708 & 1337.958 & 6.41 & 20827 \\
\mcode{T-064} & \mcode{64} & interior & \Allow & 2739.833 & 3166.542 & 3322.375 & 16.29 & 37147 \\
\mcode{T-128} & \mcode{128} & interior & \Allow & 7421.542 & 7909.542 & 8329.125 & 44.12 & 69787 \\
\mcode{T-192} & \mcode{192} & interior & \Allow & 14538.917 & 15174.291 & 15294.959 & 86.43 & 102427 \\
\mcode{T-255} & \mcode{255} & interior & \Allow & 23683.542 & 24625.958 & 25543.792 & 140.80 & 134557 \\
\mcode{T-256} & \mcode{256} & at cap & \Allow & 23514.667 & 24608.125 & 25565.791 & 139.80 & 135067 \\
\mcode{T-257} & \mcode{257} & cap+1 & \Indeterminate; \mcode{E\_PROFILE\_LIMIT} & 40.625 & 60.541 & 68.792 & --- & 27599 \\
\addlinespace[2pt]
\mcode{R-001} & \mcode{1} & interior & \Allow & 166.000 & 193.584 & 242.083 & 0.98 & 6358 \\
\mcode{R-002} & \mcode{2} & interior & \Allow & 169.083 & 193.583 & 222.625 & 1.00 & 6421 \\
\mcode{R-004} & \mcode{4} & interior & \Allow & 169.084 & 202.542 & 253.625 & 1.00 & 6547 \\
\mcode{R-008} & \mcode{8} & interior & \Allow & 169.917 & 219.250 & 240.834 & 1.00 & 6799 \\
\mcode{R-016} & \mcode{16} & interior & \Allow & 175.167 & 209.333 & 271.334 & 1.04 & 7303 \\
\mcode{R-032} & \mcode{32} & interior & \Allow & 176.292 & 198.583 & 230.917 & 1.04 & 8311 \\
\mcode{R-064} & \mcode{64} & interior & \Allow & 191.958 & 225.084 & 259.417 & 1.14 & 10327 \\
\mcode{R-128} & \mcode{128} & interior & \Allow & 214.708 & 256.833 & 390.792 & 1.27 & 14359 \\
\mcode{R-192} & \mcode{192} & interior & \Allow & 236.500 & 273.667 & 290.625 & 1.40 & 18391 \\
\mcode{R-255} & \mcode{255} & interior & \Allow & 260.250 & 312.125 & 509.125 & 1.54 & 22360 \\
\mcode{R-256} & \mcode{256} & at cap & \Allow & 258.666 & 323.125 & 410.958 & 1.53 & 22423 \\
\mcode{R-257} & \mcode{257} & cap+1 & \Indeterminate; \mcode{E\_PROFILE\_LIMIT} & 12.375 & 24.208 & 44.833 & --- & 7187 \\
\addlinespace[2pt]
\mcode{O-001} & \mcode{1} & interior & \Allow & 166.042 & 195.375 & 375.625 & 0.97 & 6340 \\
\mcode{O-002} & \mcode{2} & interior & \Allow & 169.542 & 201.875 & 314.458 & 0.99 & 6409 \\
\mcode{O-004} & \mcode{4} & interior & \Allow & 170.833 & 196.625 & 219.958 & 1.00 & 6547 \\
\mcode{O-008} & \mcode{8} & interior & \Allow & 167.333 & 203.292 & 334.959 & 0.98 & 6823 \\
\mcode{O-016} & \mcode{16} & interior & \Allow & 177.209 & 196.583 & 244.375 & 1.04 & 7375 \\
\mcode{O-032} & \mcode{32} & interior & \Allow & 185.125 & 224.042 & 264.458 & 1.08 & 8479 \\
\mcode{O-064} & \mcode{64} & interior & \Allow & 228.500 & 262.584 & 308.709 & 1.34 & 10687 \\
\mcode{O-128} & \mcode{128} & interior & \Allow & 371.833 & 455.584 & 505.959 & 2.18 & 15103 \\
\mcode{O-192} & \mcode{192} & interior & \Allow & 588.459 & 671.417 & 731.292 & 3.44 & 19519 \\
\mcode{O-255} & \mcode{255} & interior & \Allow & 881.375 & 980.167 & 1041.000 & 5.16 & 23866 \\
\mcode{O-256} & \mcode{256} & at cap & \Allow & 882.917 & 1014.750 & 1066.125 & 5.17 & 23935 \\
\mcode{O-257} & \mcode{257} & cap+1 & \Indeterminate; \mcode{E\_PROFILE\_LIMIT} & 13.917 & 35.083 & 70.417 & --- & 7691 \\
\addlinespace[2pt]
\mcode{F-001} & \mcode{1} & interior & \Allow & 146.917 & 173.458 & 227.833 & 0.87 & 6043 \\
\mcode{F-002} & \mcode{2} & interior & \Allow & 156.375 & 179.625 & 220.208 & 0.93 & 6211 \\
\mcode{F-004} & \mcode{4} & interior & \Allow & 168.750 & 196.917 & 279.084 & 1.00 & 6547 \\
\mcode{F-008} & \mcode{8} & interior & \Allow & 201.500 & 233.584 & 294.958 & 1.19 & 7219 \\
\mcode{F-016} & \mcode{16} & interior & \Allow & 299.209 & 349.792 & 411.875 & 1.77 & 8563 \\
\mcode{F-032} & \mcode{32} & interior & \Allow & 636.208 & 771.625 & 889.959 & 3.77 & 11251 \\
\mcode{F-064} & \mcode{64} & interior & \Allow & 1845.292 & 2068.708 & 2166.750 & 10.94 & 16627 \\
\mcode{F-128} & \mcode{128} & interior & \Allow & 6563.958 & 6854.166 & 6912.333 & 38.90 & 27379 \\
\mcode{F-192} & \mcode{192} & interior & \Allow & 14198.083 & 14903.833 & 15450.750 & 84.14 & 38131 \\
\mcode{F-255} & \mcode{255} & interior & \Allow & 24452.250 & 25831.958 & 26338.125 & 144.90 & 48715 \\
\mcode{F-256} & \mcode{256} & at cap & \Allow & 24723.959 & 26357.166 & 26739.833 & 146.51 & 48883 \\
\mcode{F-257} & \mcode{257} & cap+1 & \Indeterminate; \mcode{E\_PROFILE\_LIMIT} & 16.667 & 38.125 & 47.000 & --- & 16007 \\
\addlinespace[2pt]
\mcode{C-001} & \mcode{1} & interior & \Allow & 164.667 & 214.584 & 299.958 & 0.98 & 6403 \\
\mcode{C-002} & \mcode{2} & interior & \Allow & 163.417 & 190.792 & 263.209 & 0.97 & 6451 \\
\mcode{C-004} & \mcode{4} & interior & \Allow & 168.125 & 199.084 & 250.083 & 1.00 & 6547 \\
\mcode{C-008} & \mcode{8} & interior & \Allow & 179.417 & 220.542 & 345.583 & 1.07 & 6739 \\
\mcode{C-016} & \mcode{16} & interior & \Allow & 184.750 & 224.917 & 251.208 & 1.10 & 7123 \\
\mcode{C-032} & \mcode{32} & interior & \Allow & 206.958 & 234.667 & 273.917 & 1.23 & 7891 \\
\mcode{C-064} & \mcode{64} & interior & \Allow & 250.000 & 295.333 & 370.375 & 1.49 & 9427 \\
\mcode{C-128} & \mcode{128} & interior & \Allow & 336.625 & 388.750 & 429.667 & 2.00 & 12499 \\
\mcode{C-192} & \mcode{192} & interior & \Allow & 418.792 & 479.833 & 536.167 & 2.49 & 15571 \\
\mcode{C-255} & \mcode{255} & interior & \Allow & 508.583 & 594.291 & 722.208 & 3.03 & 18595 \\
\mcode{C-256} & \mcode{256} & at cap & \Allow & 503.084 & 576.625 & 706.916 & 2.99 & 18643 \\
\mcode{C-257} & \mcode{257} & cap+1 & \Indeterminate; \mcode{E\_PROFILE\_LIMIT} & 12.750 & 31.667 & 43.292 & --- & 1895 \\
\addlinespace[2pt]
\mcode{P-000} & \mcode{0} & interior & \Allow & 164.750 & 191.375 & 227.250 & 0.97 & 6328 \\
\mcode{P-001} & \mcode{1} & interior & \Allow & 163.708 & 196.250 & 334.416 & 0.96 & 6382 \\
\mcode{P-002} & \mcode{2} & interior & \Allow & 164.375 & 199.625 & 254.500 & 0.97 & 6437 \\
\mcode{P-004} & \mcode{4} & interior & \Allow & 169.792 & 208.625 & 282.875 & 1.00 & 6547 \\
\mcode{P-008} & \mcode{8} & interior & \Allow & 172.167 & 209.417 & 268.625 & 1.01 & 6767 \\
\mcode{P-016} & \mcode{16} & interior & \Allow & 183.625 & 210.625 & 279.666 & 1.08 & 7207 \\
\mcode{P-032} & \mcode{32} & interior & \Allow & 195.834 & 232.833 & 295.667 & 1.15 & 8087 \\
\mcode{P-064} & \mcode{64} & interior & \Allow & 225.958 & 260.333 & 422.958 & 1.33 & 9847 \\
\mcode{P-128} & \mcode{128} & interior & \Allow & 286.500 & 331.375 & 570.417 & 1.69 & 13367 \\
\mcode{P-192} & \mcode{192} & interior & \Allow & 349.583 & 452.041 & 586.250 & 2.06 & 16887 \\
\mcode{P-256} & \mcode{256} & study ceiling & \Allow & 414.875 & 611.833 & 680.875 & 2.44 & 20407 \\
\addlinespace[2pt]
\mcode{J-002} & \mcode{2} & interior & \Allow & 96.041 & 130.542 & 166.125 & 0.82 & 11294 \\
\mcode{J-003} & \mcode{3} & interior & \Allow & 113.167 & 177.333 & 222.792 & 0.96 & 14866 \\
\mcode{J-004} & \mcode{4} & interior & \Allow & 117.584 & 142.083 & 164.916 & 1.00 & 18526 \\
\mcode{J-005} & \mcode{5} & interior & \Allow & 127.708 & 160.666 & 197.292 & 1.09 & 22274 \\
\mcode{J-006} & \mcode{6} & interior & \Allow & 137.917 & 170.584 & 194.750 & 1.17 & 26110 \\
\mcode{J-007} & \mcode{7} & interior & \Allow & 150.417 & 202.583 & 235.875 & 1.28 & 30034 \\
\mcode{J-008} & \mcode{8} & at cap & \Allow & 161.125 & 207.958 & 246.167 & 1.37 & 34046 \\
\mcode{J-009} & \mcode{9} & cap+1 & \Indeterminate; \mcode{E\_PROFILE\_LIMIT} & 11.584 & 21.250 & 33.833 & --- & 29231 \\
\addlinespace[2pt]
\mcode{Q-Q0} & \mcode{Q0} & interior & \Allow & 1266.125 & 1559.541 & 1801.042 & 9.29 & 25206 \\
\mcode{Q-Q1} & \mcode{Q1} & interior & \Allow & 719.250 & 940.583 & 1085.875 & 5.28 & 25300 \\
\mcode{Q-Q2} & \mcode{Q2} & interior & \Allow & 306.583 & 392.667 & 514.625 & 2.25 & 25492 \\
\mcode{Q-Q3} & \mcode{Q3} & interior & \Allow & 187.958 & 235.625 & 270.542 & 1.38 & 25588 \\
\mcode{Q-Q4} & \mcode{Q4} & interior & \Allow & 154.375 & 215.916 & 268.500 & 1.13 & 25588 \\
\mcode{Q-Q5} & \mcode{Q5} & interior & \Allow & 136.333 & 174.458 & 210.958 & 1.00 & 25684 \\
\addlinespace[2pt]
\mcode{A-001} & \mcode{1} & interior & \Allow & 110.083 & 139.958 & 198.250 & 0.53 & 11570 \\
\mcode{A-002} & \mcode{2} & interior & \Allow & 141.375 & 177.417 & 230.875 & 0.68 & 12408 \\
\mcode{A-004} & \mcode{4} & interior & \Allow & 207.208 & 261.125 & 296.916 & 1.00 & 14084 \\
\mcode{A-008} & \mcode{8} & interior & \Allow & 361.584 & 453.083 & 558.375 & 1.75 & 17436 \\
\mcode{A-016} & \mcode{16} & interior & \Allow & 713.708 & 830.958 & 933.417 & 3.44 & 24140 \\
\mcode{A-032} & \mcode{32} & interior & \Allow & 1652.500 & 1888.833 & 2683.375 & 7.98 & 37548 \\
\mcode{A-064} & \mcode{64} & interior & \Allow & 4605.125 & 4967.000 & 5153.875 & 22.22 & 64364 \\
\mcode{A-128} & \mcode{128} & at cap & \Allow & 14091.709 & 15044.250 & 15441.500 & 68.01 & 117996 \\
\mcode{A-129} & \mcode{129} & cap+1 & \Indeterminate; \mcode{E\_PROFILE\_LIMIT} & 35.000 & 49.417 & 63.708 & --- & 28499 \\
\addlinespace[2pt]
\mcode{ITC-EXACT} & \mcode{ITC-EXACT} & at cap & \Deny; \mcode{E\_JOIN\_COMPOSITION} & 11533.292 & 13092.250 & 13952.334 & --- & 33719 \\
\mcode{ITC-OVER} & \mcode{ITC-OVER} & over guard & \Indeterminate; \mcode{E\_PROFILE\_LIMIT} & 293.583 & 355.500 & 575.709 & --- & 33473 \\
\addlinespace[2pt]
\mcode{ITO-OVER} & \mcode{ITO-OVER} & over guard & \Indeterminate; \mcode{E\_PROFILE\_LIMIT} & 3236.542 & 3667.916 & 4161.000 & --- & 26652 \\
\mcode{ITO-BELOW} & \mcode{ITO-BELOW} & below guard & \Allow & 1340.875 & 1675.125 & 1742.041 & --- & 27236 \\
\addlinespace[2pt]
\mcode{IPE-EXACT} & \mcode{IPE-EXACT} & at cap & \Allow & 217.750 & 261.791 & 284.042 & --- & 38810 \\
\mcode{IPE-OVER} & \mcode{IPE-OVER} & over guard & \Indeterminate; \mcode{E\_PROFILE\_LIMIT} & 14.125 & 30.750 & 40.833 & --- & 30638 \\
\end{longtable}
\noindent\footnotesize \textit{Memory interpretation.} One isolated subprocess snapshot was retained per cell. Heap/RSS deltas compare pre-call, immediate post-call, and post-GC snapshots and may be negative; they are not peak-allocation measurements. Process \mcode{maxRSS} is a process-lifetime high-water mark that includes fixture setup.
\endgroup

\end{document}